%% file: preprint.tex
\documentclass{applemlr}

\input{apple_preamble}

\usepackage{thm-restate}
\usepackage{algorithm}
\usepackage{algpseudocode}
\creflabelformat{equation}{#2\textup{(#1)}#3}

\definecolor{rowhl}{HTML}{E8EBF0}
\newcommand{\hlrow}{\rowcolor{rowhl}}
\newcolumntype{Y}{>{\centering\arraybackslash}X}

\theoremstyle{definition}

\definecolor{propbg}{RGB}{241,239,248}
\tcbset{
  propboxstyle/.style={
    enhanced, breakable,
    colback=propbg, colframe=propbg,
    boxrule=0pt, arc=8pt,
    left=10pt, right=10pt, top=8pt, bottom=8pt,
    before skip=10pt, after skip=10pt,
  }
}
\tcolorboxenvironment{proposition}{propboxstyle}

\newtcolorbox{propbox}{propboxstyle, before upper={}}

\newtcolorbox{eqbox}{propboxstyle, breakable=false, top=2pt, bottom=2pt,
  before upper={%
    \setlength{\abovedisplayskip}{0pt}\setlength{\belowdisplayskip}{0pt}%
    \setlength{\abovedisplayshortskip}{0pt}\setlength{\belowdisplayshortskip}{0pt}}}

\newcommand{\takeawayname}{Takeaway}
\newtcolorbox{takeaway}[1][\takeawayname]{propboxstyle, breakable=false,
  before upper={\textbf{#1.}\ }}

\newcommand{\ie}{\emph{i.e.},~}
\newcommand{\eg}{\emph{e.g.},~}

\newcommand{\cut}[1]{}

\input{math}

\title{On the Design Space of Continuous Diffusion Language Models}

\author[1,*]{Oscar Davis}
\author[2,*]{Daan Roos}
\author{Alice Bizeul}
\author{Amitis Shidani}
\author{Victor Turrisi}
\author{Louis B\'ethune}
\author{Pierre Ablin}
\author{Marco Cuturi}

\affiliation{Apple}
\affiliation[1]{University of Oxford}
\affiliation[2]{University of Amsterdam}

\abstract{Continuous diffusion language models (CDLMs) promise to bring the benefits of continuous flow
modelling to language, including accelerated sampling and controllable generation.
The field has grown rapidly, and the capabilities of CDLMs have begun to rival those of discrete diffusion.
Yet, unlike for continuous modalities such as images or videos, our grasp of the levers that matter when training CDLMs remains limited.
In this work, we focus on two of these levers, namely self-conditioning and embedding geometry, and shed light on some of the mechanisms behind them, which in turn suggest cheaper and more effective designs.
We begin by showing that self-conditioning corrects a train/inference mismatch by exposing the model to inference-like states that are otherwise rare during training. These gains can be recovered without an extra forward pass by a two-level interpolant, coined $\mathrm{DoubleCat}$, which improves TinyGSM accuracy more than threefold. Self-conditioning also aids in recognising tokens dominated by large-norm Gaussians---an issue known as token identifiability. To address it directly, we introduce discrete reveals to $\mathrm{DoubleCat}$, yielding $\mathrm{SemiCat}$, which achieves discrete diffusion-level performance on TinyGSM, reaching $49.94\%$ with standard training hyperparameters and a tuned sampler. At the $1.7$B scale, $\mathrm{SemiCat}$ outperforms discrete diffusion on four of six likelihood-based evaluations, and surpasses one-hot flows on unconditional generation.
On embedded spaces, we find that cross-entropy drives embeddings towards a geometry whose decoding error rate closely matches the one-hot one.
We briefly explore a range of further options, such as hybrid corruption processes and
noise schedules, including asynchronous ones.
We believe this exploration can help the community better understand and advance the design space of CDLMs.}

\metadata[Correspondence]{\sffamily Oscar Davis: \href{mailto:oscar.davis@cs.ox.ac.uk}{oscar.davis@cs.ox.ac.uk}, Marco Cuturi: \href{mailto:m_cuturi@apple.com}{m\_cuturi@apple.com}}
\metadata[Contributions]{\sffamily $^*$Work done while at Apple}
\date{\sffamily\today}

\begin{document}

\maketitle

\input{section/introduction}
\input{section/background}
\input{section/method}
\input{section/experiment}
\input{section/related}
\input{section/conclusion}

\bibliographystyle{apalike}
\bibliography{ref}

\clearpage
\appendix
\crefalias{section}{appendix}
\crefalias{subsection}{appendix}
\crefalias{subsubsection}{appendix}
\input{section/appendix}

\end{document}

%% file: apple_preamble.tex
\usepackage{amsmath}
\usepackage{enumerate}
\usepackage{algorithm}
\usepackage{algpseudocode}
\usepackage{amsfonts}
\usepackage{amsthm}
\usepackage{cleveref}
\usepackage{diagbox}
\usepackage{colortbl}
\usepackage{amssymb}
\usepackage{xspace}
\usepackage{wrapfig}
\usepackage{adjustbox}
\usepackage{tabularx}
\usepackage{booktabs}
\usepackage{mathtools}
\usepackage{tikz}
\usepackage{enumitem}
\usepackage{silence}
\usepackage{dsfont}
\usepackage[table]{xcolor}
\usepackage[dvipsnames]{xcolor}
\usepackage{multirow}
\usepackage{makecell}
\usepackage{xfakebold}

\definecolor{textgray}{HTML}{6E6E73}
\usetikzlibrary{positioning, calc}
\usetikzlibrary{decorations.pathmorphing}

\makeatletter
\patchcmd{\wrong@fontshape}{\@gobbletwo}{}{}{}
\makeatother
\numberwithin{equation}{section}
\makeatletter
\AtBeginDocument{
  \urlstyle{sf}
  
}
\makeatother

\definecolor{light}{RGB}{125, 125, 125}
\crefname{tcb@cnt@pbox}{code}{code}
\Crefname{tcb@cnt@pbox}{Code}{Code}
\crefname{assumption}{assumption}{assumption}
\Crefname{assumption}{Assumption}{Assumptions}

\newtcolorbox[auto counter]{pbox}[2][]{
  colback=white,
  title=Code~\thetcbcounter: #2,
  #1,fonttitle=\sffamily,
  fontupper=\sffamily,
  arc=2pt,
  colframe=bgcolor,
  coltitle=fgcolor,
  colbacktitle=bgcolor,
  toptitle=0.25cm,
  bottomtitle=0.125cm
}

\makeatletter
\newcommand\applefootnote[1]{%
  \begingroup
  \renewcommand\thefootnote{}%
  \renewcommand\@makefntext[1]{\noindent##1}%
  \footnote{#1}%
  \addtocounter{footnote}{-1}%
  \endgroup
}
\makeatother

\definecolor{cverbbg}{gray}{0.90}

%% file: math.tex
\renewcommand{\mapsto}{\longmapsto}
\renewcommand{\to}{\longrightarrow}

\newcommand{\alphacat}{\ensuremath{\alpha\text{-}\mathrm{Cat}}\xspace}
\newcommand{\doublecat}{\ensuremath{\mathrm{DoubleCat}}\xspace}
\newcommand{\semicat}{\ensuremath{\mathrm{SemiCat}}\xspace}
\newcommand{\sflm}{\ensuremath{\mathbb{S}}-FLM\xspace}
\newcommand{\genppl}{Gen.\ PPL\xspace}

\newcommand{\edr}{\gamma^{\mathrm{DER}}}

\newcommand{\R}{\mathbb{R}}
\newcommand{\N}{\mathbb{N}}
\newcommand{\Nstar}{\mathbb{N}^\star}

\newcommand{\stopgrad}{\mathrm{sg}}

\newcommand{\argmax}{\mathrm{argmax}}
\newcommand{\argmin}{\mathrm{argmin}}
\newcommand{\indic}{\mathbf{1}}
\newcommand{\vx}{\mathbf{x}}
\newcommand{\vy}{\mathbf{y}}
\newcommand{\vv}{\mathbf{v}}
\newcommand{\vzero}{\mathbf{0}}

\newcommand{\E}{\mathbb{E}}
\newcommand{\KL}[2]{\mathrm{KL}\left(#1\,\middle\|\,#2\right)}
\newcommand{\Prob}{\mathbb{P}}

\newcommand{\cov}{\mathrm{cov}}
\newcommand{\MI}{\mathrm{I}}
\newcommand{\distortion}{\mathcal{D}}
\newcommand{\mmse}{\mathrm{MMSE}}

\newcommand{\law}{\mathrm{Law}}
\newcommand{\unif}{\mathcal{U}}
\newcommand{\normal}{\mathcal{N}}
\newcommand{\cat}{\mathrm{Cat}}
\newcommand{\ber}{\mathrm{Ber}}

\newcommand{\matid}{\mathrm{I}}
\newcommand{\pdata}{\ensuremath{p_\mathrm{data}}}
\newcommand{\softmax}{\mathrm{softmax}}

\newcommand{\SNR}{\mathrm{SNR}}

\newcommand{\norm}[1]{\left\lVert#1\right\rVert}
\newcommand{\abs}[1]{\left\lvert#1\right\rvert}
\newcommand{\intrange}[2]{\{#1, \dots, #2\}}

\newcommand{\loss}{\mathcal{L}}
\newcommand{\emb}{\mathsf{emb}}
\newcommand{\embspace}{\mathcal{E}}
\newcommand{\vocabsize}{{\abs{V}}}
\newcommand{\pcatunif}{\frac{\mathbf{1}}{\vocabsize}}

\newcommand{\embed}{\mathbf{e}}
\newcommand{\embhead}{\embed_\theta}

\newcommand{\psc}{\ensuremath{p_\mathrm{sc}}}
\newcommand{\SC}{\hat{x}_1}

\newcommand{\SCvarl}{\hat{X}_1^\ell}

\newcommand{\CITATION}[1][]{%
  \textcolor{red}{\textbf{[CITATION\if\relax\detokenize{#1}\relax\else: #1\fi]}}%
}

\newcommand{\masktok}{\mathsf{m}}

%% file: section/introduction.tex
\section{Introduction}
Continuous diffusion language models (CDLMs) hold several promises for text generation. Ideally, they should port to the world of language modelling all of the key properties of the continuous toolbox, such as fast and controllable generation, consistency models~\citep{consistencymodels,meanflows,howToBuildConsistencyModel}, or stochastic optimal control-based steering~\citep{metaflowmaps,diamondmaps,tiltmatching}.
Yet, some of the seminal advances that have managed to shrink the performance gap between continuous and discrete diffusion methods have relied instead on adaptations that deviate significantly from the continuous diffusion literature. One such technique is \emph{self-conditioning}~\citep{analogbits}, which augments inputs using previous outputs, costing an additional forward pass during training. \citet{fmlmstar} interpret self-conditioning as a fixed-point iteration, and \citet{discreteDataFail} understand it through the lens of multi-modality. Yet, it remains unclear which deficiencies of CDLMs it corrects, and whether these can be addressed more cheaply. Another innovative line of development lies in embedded methods~\citep{hypersphere,vmfFlow,elf,langflow,cobit}, which alleviate the computational and memory cost of one-hot methods~\citep{cfm,fmlm,dfm}. They rely on carefully chosen noise schedules. However, it remains unclear how such schedules interact with the learnt embeddings.

\textbf{In this work}, we show first that self-conditioning addresses a discrepancy between inference- and train-time trajectories (\Cref{ssec:traj-ambiguity}), and increases the density of otherwise under-sampled regions at train-time (consistent with~\citealt{discreteDataFail}). This gain can be recovered by a simple two-level interpolant, called~\doublecat, which requires no additional forward pass of the model, and which more than triples the accuracy on TinyGSM~\citep{tinygsm} from $4.42\%$ to $\mathbf{16.00\%}$. Moreover, we find that self-conditioning helps in \emph{token identifiability}~\citep{candi} (\Cref{ssec:self-cond-embedded}): given the noisy state, the model cannot easily identify the expected endpoint of the trajectory, due to the norm mismatch between the noise and the token's embedding, whereas self-conditioning is supported on the convex hull formed by the embeddings. Replacing that signal with a better representation of the state recovers a substantial part of the performance gain at no additional training cost.
\looseness=-1

We address the token identifiability issue in \doublecat (\Cref{ssec:token-identifiability}). By introducing discrete stochastic reveals akin to CANDI~\citep{candi}, the model is able to condition correctly on parts of its context, helping identifiability. This method, called \semicat (semi-continuous categorical flow), further boosts the performance up to $\mathbf{32.90\%}$, \emph{reaching the discrete diffusion baseline of about $33.74\%$}. Using standard hyperparameters on the task, and a confidence-based sampler, \semicat reaches $\mathbf{49.94\%}$ \emph{without any self-conditioning}. At the $1.7$B scale, \semicat outperforms discrete diffusion models on $4$ out of $6$ likelihood-based benchmarks and matches them on a fifth one; it also surpasses one-hot flows in unconditional generation quality (\Cref{sec:large-experiment}).

We also cover embedded methods. Because of architectural decisions, one-hot flows effectively operate on an embedded space with an anisotropic noise defined by the embeddings' geometry (\Cref{ssec:one-hot-is-embed}), on which we can \emph{directly train}. We demonstrate that embeddings may not be trained with cross-entropy naively, as it maximises their spread. \emph{This itself induces a one-hot-like geometry on the embedded space when considering its corresponding decoding error rate} (\Cref{ssec:embed-is-one-hot}). This may explain why carefully crafted schedules are needed. We argue that training should update the geometry of all embeddings jointly.

Lastly, we briefly explore other avenues such as asynchronous schedules, dynamic time allocation, hybrid methods, and acceleration, thus covering a large part of the design space of CDLMs. We believe that our results on our benchmarks show the promise of continuous and hybrid methods, and unveil important avenues left to explore to understand continuous diffusion for discrete data.

%% file: section/background.tex
\section{Background} \label{sec:background}
Let $V \subset \N$ be the finite vocabulary of tokens, and $\vx = (\vx^1, \ldots, \vx^L) \in V^L$ a sequence of length $L \in \Nstar$. Our goal is to model some data distribution, \pdata, given independent and identically distributed samples therefrom. Model parameters are denoted by $\theta$. A uniform distribution over $I$ is denoted by $\unif(I)$. We defer to~\Cref{app:discrete-diff-background} a brief presentation of discrete diffusion models.

\paragraph{Flow-based generative modelling.} A common option for the modelling of \pdata~is through the dynamic transport of measure~\citep{flowmatching,reflow,stochasticInterpolants}. Given a simple starting distribution, $p_0$, and a target distribution $p_1 \equiv \pdata$, the aim is to create a probability path between the two, $(p_t)_{t\in[0,1]}$, via an ODE: $\dot x_t = v_t(x_t),\text{ with } x_0 \sim p_0$.

Let $I:[0,1]\times\R^d\times \R^d\to\R^d$ be an interpolant of the form $I_t(X_0, X_1) = \alpha_t X_0 + \beta_t X_1$, for some differentiable $\alpha$ and $\beta$, $X_0\sim p_0$, $X_1\sim p_1$, and such that $I_0(X_0, X_1) = X_0$ and $I_1(X_0, X_1) = X_1$. By choosing $v_t(x) = \E[\dot I_t \mid I_t = x]$, then $\law(x_t) = \law(I_t)$ for any $t$, making $x_t$ available \emph{simulation-free} through $I_t$. It thus reduces the problem of generative modelling to the regression:
\begin{equation}
    \loss(\theta) = \E_{X_0, X_1, t\sim \unif(0,1)}\norm{v^\theta_t(I_t) - \dot I_t}_2^2.
\end{equation}
Sampling is done by integrating $v^\theta$. When $\beta_t \equiv 1- \alpha_t$, we write $I_t(X_0, X_1) = (1-\gamma_t) X_0 + \gamma_t X_1$.

\begin{figure}[t]
\centering
\begin{minipage}{0.325\textwidth}
    \centering
    \includegraphics[width=\linewidth]{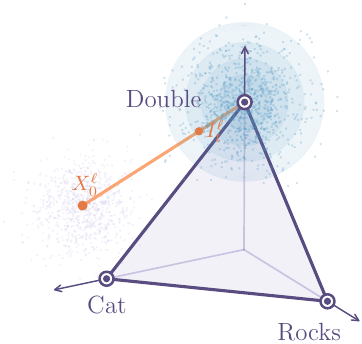}\\
    \small (a) One-hot flow, train-time.
\end{minipage}\hfill
\begin{minipage}{0.325\textwidth}
    \centering
    \includegraphics[width=\linewidth]{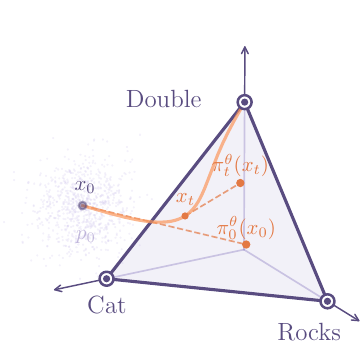}\\
    \small (b) One-hot flow, inference.
\end{minipage}\hfill
\begin{minipage}{0.325\textwidth}
    \centering
    \includegraphics[width=\linewidth]{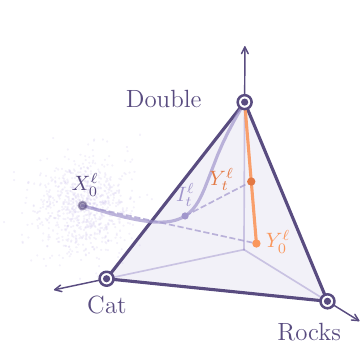}\\
    \small (c) \doublecat, train time.
\end{minipage}
\caption{\textbf{(a)} One-hot flows are always exposed to nearly clean samples for $t\approx 1$, and only track the clean target, $X_1^\ell$. \textbf{(b)} At inference, they track instead the dynamic predictions of their denoiser, $\pi^\theta$, which always lies on the interior of the simplex. \textbf{(c)} \doublecat addresses this issue at train-time by simply adding an additional corruption process on the output space of its denoiser, the simplex.}
\label{fig:doublecat}
\end{figure}

\paragraph{Flow-based language modelling.} Flow modelling has also been adapted to the setting where \pdata~is a discrete distribution in $V^L$~\citep{vfm,cfm,fmlm,dfm}. One choice of modelling is that of \emph{one-hot flows}, where the final distribution is one-hot encoded into $\R^{\vocabsize}$ (\ie $X_1^\ell = e_{\vx^\ell}$). Using the endpoint parametrisation of the generating vector field, $v_t(x) = \dot \gamma_t\frac{\E[X_1\mid I_t = x] - x}{1-\gamma_t}$, it is possible to train using a cross-entropy loss instead:
\begin{equation}
    \label{eq:one-hot-flow-ce}
    \loss(\theta) = -\E_{X_0, X_1, t}\sum_{\ell=1}^L \log \left\langle X_1^\ell, \pi^{\theta,\ell}_t(I_t)\right\rangle,
\end{equation}
where $\pi^\theta:[0,1]\times\R^{L\times\vocabsize}\to\R^{L\times \vocabsize}$ is a denoising network. In this work, we also consider the possibility of evolving the flow on an embedded space, which is first embedded through a trainable codebook~\citep{hypersphere,vmfFlow,elf,langflow}. We denote by $\embspace \subseteq \R^d$ the embedded space, with $d \leq \vocabsize$, and by $\emb:V\to\embspace$ the embedding function. For one-hot values $e_v$ (\eg $X_1^\ell$), we extend the mapping as $\emb(e_v) = \emb(v)$; similarly, the embedding of a sequence is the element-wise embedding. Let $W$ be the matrix such that $W_v = \emb(v)$.

\paragraph{Time schedule.} The choice of time schedule is crucial for one-hot flows~\citep{fmlm,dfm}. Following~\citet{scalingcfm,dfm}, we refer to and use by default for one-hot flows a \emph{mixture schedule}, $\gamma_t = \lambda \edr_t + (1-\lambda) t$, where $\edr$ is the \emph{decoding error rate} (DER) schedule~\citep{candi,duo,fmlm}, with $\lambda=0.75$. $\edr$ makes uniform progress in the DER~\citep[eq. (25)]{fmlm}:
\begin{equation}
\label{eq:base-edr}
P_e(t) \coloneqq \frac{1}{L} \sum_{\ell=1}^{L} \Prob(\argmax~I_t^\ell \neq \argmax \,X_1^{\ell}).
\end{equation}

\paragraph{Hybrid methods.}~\citet{candi} introduce a hybrid flow, CANDI, that continuously evolves each state of the sequence until some randomly sampled unmasking time, $\tau^\ell$, per position. Formally, for each $\ell$, we let i.i.d.\ $\tau^\ell \sim \unif(0,1)$; for $t < \tau^\ell$, $I_t^\ell(X_0, X_1) = (1-\gamma_t) X_0^\ell + \gamma_t X_1^\ell$, and $I_t^\ell(X_0, X_1) = X_1^\ell$ otherwise. The clean token is directly revealed as a one-hot encoding at $\tau^\ell$.\footnote{Note the reversal in time convention to match the flow matching one, and note that we here ignore the difference in interpolants incurred by typical diffusion schedules.}

\citet{fmlmplus} introduce a similar construction via Flow Maps~\citep{howToBuildConsistencyModel,flowMapMatching}: FMLM+ adds to one-hot flows time $t=0$ random reveals of the sequence's tokens at positions $\mathcal{C}$, and learns the flow map from $X_0$ to $X_1$, \ie $X_{0,1}(X_0) = X_1$, at positions $\intrange{1}{L}\setminus \mathcal{C}$. The sampler goes back and forth between times $0$ and $1$, similar to $\gamma$-sampling from~\citet{kim2024consistencytrajectorymodels}, committing a subset of tokens based on the model's confidence at each round.

\paragraph{Plaid.} Introduced in~\citet{plaid} and recently revisited in~\citet{replaid}, Plaid is a Variational Diffusion Model~\citep[VDM;][]{vdm} for language modelling, trained by minimising the negative evidence lower bound (NELBO):
\begin{equation}
\label{eq:plaid-elbo}
\begin{array}{ll}
-\log p_\theta(\vx) \leq &\loss_{\mathrm{NELBO}}(\theta) \coloneqq \KL{q(X_0\mid \vx)}{p(X_0)} - \E_{I_1\mid \vx}\sum_{\ell=1}^L \log \left\langle e_{\vx^\ell}, \pi^{\theta,\ell}_1(I_1)\right\rangle \\
& + \frac{1}{2} \E_{t, I_t\mid \vx}\left[\mathrm{SNR'}(t)\norm{\embhead(I_t, t) - \emb(\vx)}^2\right],
\end{array}
\end{equation}
where $\embhead:\embspace^L \times [0,1] \to \embspace^L$ is given by $\embhead^\ell:(x, t)\mapsto \pi^{\theta,\ell}_t(x)W$ for each $\ell$, $p \equiv p_0$ (Gaussian prior), and $\SNR'$ being the derivative of the \emph{signal-to-noise ratio}, $\SNR:[0,1)\to\R, t\mapsto {\beta_t^2}/{\alpha_t^2}$. The schedules are learnt, following VDM.

\paragraph{Self-conditioning.} Self-conditioning has been shown to be extremely beneficial for CDLMs~\citep{analogbits,elf,langflow,cobit}. It extends the denoiser to accept an additional input, $\tilde\pi^\theta:[0,1]\times \R^d\times \R^d\to\R^d$, and feeds the model its outputs from the previous inference step when generating samples. At train time, it is provided to the model with a probability $\psc$: let $C\sim\ber(\psc)$, then $\SC = C\cdot \stopgrad(\tilde\pi_t^\theta(I_t, \vzero))$ is the self-conditioning signal, where $\stopgrad$ is the stop-gradient operator, training then uses $\tilde \pi_t^\theta(I_t, \SC)$.

Concretely, the cross-entropy loss becomes:
\begin{equation}
    \label{eq:ce-with-sc}
    \loss(\theta) = -\E_{X_0, X_1, t, \SC}\sum_{\ell=1}^L \log \left\langle X_1^\ell, \tilde\pi^{\theta,\ell}_t(I_t, \SC)\right\rangle.
\end{equation}

%% file: section/method.tex
\section{Enhancing one-hot flows}
\label{sec:one-hot}
\begin{wrapfigure}{r}{0.42\textwidth}
    \vspace{-4.5em}
    \centering
    \includegraphics[width=0.42\textwidth]{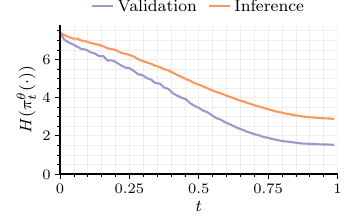}
    \caption{Entropy of a one-hot flow's denoiser, $\pi^\theta$, along validation set and inference trajectories.}
    \label{fig:catflow-decreasing-entropy}
    \vspace{-2em}
\end{wrapfigure}
In this section, we identify two issues that one-hot flows suffer from. First, they lack trajectory ambiguity: they are trained on clean sequences at $t \approx 1$, which mismatches the statistics of their inference trajectories. Second, they are unable to identify the tokens that they are generating, which we attribute to the high norm Gaussian noise. We show two simple interventions, \doublecat and \semicat, respectively, both of which drastically improve the resulting performance. 

\paragraph{Benchmark.} We mostly rely on TinyGSM~\citep{tinygsm}: the model is trained on generating Python solutions to grade school mathematical problems, and is then evaluated on GSM8K~\citep{gsm8k}. To understand failure modes of trained models, we measure the percentage of programs that run, the percentage that have name errors (referencing non-existent variables), and the final accuracy. We use the settings of~\citet{puma}. Finally, retained methods are also evaluated on the standard OpenWebText (OWT)~\citep{openwebtext} unconditional generation task on which we report the generative perplexity (\genppl) as measured by GPT-2~\citep{gpt2} alongside the entropy of generated samples, reported as Pareto fronts. See~\Cref{app:hyperparameters} for more detail. As baselines, we provide masked discrete diffusion (MDM) and uniform-state diffusion (USDM), for discrete diffusion models, and FMLM+~\citep{fmlmplus} and CANDI~\citep{candi} for hybrid methods. We retrain all baselines under the same settings, and report those numbers.

\subsection{Trajectory ambiguity}
\label{ssec:traj-ambiguity}
For a one-hot flow, near $t \approx 1$, the model is only trained on sequences in a very small radius around the final endpoint: $I_t$ is a highly concentrated Gaussian with covariance $(1-\gamma_t)^2 \matid$ and mean $\gamma_t X_1$. Moreover, at inference time, the model follows the learnt denoiser, $\pi^\theta$, which has non-zero entropy for any $t$ (\Cref{fig:catflow-decreasing-entropy}), and therefore never points towards a clean vertex of the simplex. This creates a mismatch between inference and training, where the model, instead of following a prediction for $X_1^\ell$, follows the path defined by the denoiser $(\pi_t^\theta(x_t))_{t\in[0,1]}$ on the interior of the simplex. On the other hand, self-conditioning provides a signal with the same support as the denoiser, which perturbs the Gaussian inputs towards inference-like regions~\citep{discreteDataFail}.

\paragraph{Double continuous corruption.} We introduce \doublecat, a doubly categorical flow, with a corruption process on the simplex, tracked by a normal Gaussian one. Formally, we define
\begin{equation}
    \begin{aligned}
        Y_t^\ell &= (1-\kappa_t) Y_0^\ell + \kappa_t X_1^\ell\\
        I_t^{\mathsf{dc},\ell}(X_0, X_1) &= (1-\gamma_t)X_0^\ell + \gamma_t Y_t^\ell,
    \end{aligned}
\end{equation}
for some $Y_0^\ell$ prior, and $\kappa$ a differentiable schedule, with $\kappa_0 = 0$ and $\kappa_1 = 1$. Equivalently, it can be written as
\begin{equation}
    \label{eq:dc-interpolant}
    I_t^{\mathsf{dc},\ell}(X_0, X_1) = (1-\gamma_t)X_0^\ell + \gamma_t(1-\kappa_t)Y_0^\ell + \gamma_t \kappa_t X_1^\ell.
\end{equation}
Note that this forms a valid probability path between $p_0$ and $p_1$. This formulation comes from the observation that the inference trajectory $(\hat I_t^\ell)_{t\in[0,1]}$ (induced by the denoiser-driven ODE) can be written as
\begin{equation}
    \hat I_t^\ell = (1-\gamma_t) X_0^\ell + \gamma_t \hat Y_t^\ell,
\end{equation}
for $(\hat Y_t^\ell)_{t\in[0,1]}$ a trajectory over the simplex corresponding to a weighted average of the predictions of the chosen denoiser; the derivation is given in~\Cref{app:doublecat-proof}. $\hat Y_t^\ell$ being a convex combination of predictions, by the concavity of the entropy, $H(\hat Y_t^\ell)$ is at least the weighted average entropy of the denoiser along the path, which is positive at all $t$ (\Cref{fig:catflow-decreasing-entropy}). For \doublecat, if we choose $Y_0^\ell$ to have mean $\pcatunif$, then the entropy evolves roughly from $H(\pcatunif) = \log\vocabsize$ to $H(Y_1^\ell) = 0$, thus remaining positive until late, like $\hat Y_t^\ell$. \doublecat therefore trains on inference-like states. Pure one-hot flow training is recovered for $\kappa_t \equiv 1$, which, in contrast, puts no mass on simplex components with non-zero entropy, since $H(X_1^\ell) = 0$.

For inference, we use the same denoiser for both parts of the flow as an approximation. We ablate the choice of prior in~\Cref{tab:doublecat-prior}, and find the uniform distribution over the simplex to work best. We also ablate the schedule $\kappa$ in~\Cref{tab:doublecat-schedule}, and retain $\kappa:t\mapsto \sqrt{t}$.~\Cref{fig:doublecat} illustrates our intuition.

\begin{table}[t]
\centering
\small
\setlength{\tabcolsep}{4pt}
\begin{tabularx}{\textwidth}{@{}Xcccccc@{}}
\toprule
\multirow{2}{*}[-2pt]{\textbf{Method}} & \multicolumn{4}{c}{\textbf{Accuracy} (\%, $\uparrow$)} & \multicolumn{2}{c}{\textbf{Best} (\%)} \\
\cmidrule(lr){2-5}\cmidrule(lr){6-7}
& $T = 1$ & $T = 0.5$ & $T = 0.2$ & $\mathrm{top}\text{-}k\text{=}1$ & Runs ($\uparrow$) & Name err. ($\downarrow$) \\
\midrule
\emph{Discrete diffusion baselines}\\
\quad MDM & $8.79$ & $21.68$ & $27.14$ & $28.28$ & $88.63$ & $11.22$ \\
\quad USDM & $\mathbf{12.43}$ & $\mathbf{29.34}$ & $\mathbf{33.06}$ & $\mathbf{33.74}$ & $\mathbf{91.96}$ & $\mathbf{7.81}$ \\
\midrule
\emph{One-hot based}\\
\quad Base one-hot flow & $3.28$ & $4.42$ & $3.89$ & $3.64$ & $50.91$ & $48.07$ \\
\qquad with self-cond. & ${6.44}$ & ${6.75}$ & ${8.34}$ & ${6.97}$ & ${61.79}$ & ${37.60}$ \\
\hlrow \quad \doublecat & $9.70$ & $13.50$ & $14.18$ & $14.78$ & $63.61$ & $36.16$ \\
\hlrow \qquad Base one-hot sampling & $\mathbf{11.45}$ & $\mathbf{15.85}$ & $\mathbf{16.00}$ & $\mathbf{15.92}$ & $\mathbf{69.67}$ & $\mathbf{30.02}$ \\
\midrule
\emph{Hybrid methods, one-hot}\\
\quad CANDI~\citep{candi} & $7.20$ & $9.55$ & $9.63$ & $8.72$ & $69.22$ & $29.26$ \\
\quad FMLM+~\citep{fmlmplus} & $16.83$ & $18.88$ & $18.27$ & $18.50$ & $90.52$ & $9.02$ \\
\hlrow \quad \semicat & $17.51$ & $29.34$ & $32.22$ & $32.90$ & $89.23$ & ${10.39}$ \\
\hlrow \qquad with standard hyperparameters$^\dagger$ & $21.08$ & $35.41$ & $37.68$ & $38.67$ & $90.75$ & $9.10$ \\
\hlrow \qquad $+$ InfoNoise~\citep{raya2026noiseschedulinginformationguidedallocation} & $\mathbf{21.46}$ & $\mathbf{36.24}$ & $41.09$ & $38.06$ & $86.58$ & $13.04$ \\
\hlrow \qquad $+$ \Cref{alg:semicat-tuned-sampling} {\footnotesize($\rho/N=0.75$, $\lambda=0.25$, $\eta_\mathrm{max}=0.25$)} & -- & -- & $\mathbf{49.73}$ & $\mathbf{49.94}$ & $\mathbf{95.36}$ & $\mathbf{4.38}$ \\
\bottomrule
\end{tabularx}
\caption{TinyGSM results at $256$ NFEs. Accuracy is reported per sampler; the last two columns report the best value attained over those samplers. $T$ denotes the temperature applied to $\pi^\theta$. The models are trained with batch size $256$ for $200$k steps. $^\dagger$ refers to standard hyperparameters for the TinyGSM benchmark in the literature (including batch size of $512$, $250$k steps). See~\Cref{app:hyperparameters} for more details. }
\label{tab:tinygsm-doublecat}
\end{table}

\paragraph{Results.} The results on TinyGSM are available in~\Cref{tab:tinygsm-doublecat}. \doublecat's performance improves by more than $3\times$ over one-hot flows. Discarding the $Y_t$ term at inference---\ie sampling as in a usual one-hot flow---further improves the performance: up to over $4\times$ (with $T=0.2$), with an accuracy reaching $16.00\%$, name error percentages lower by $18$ points, and valid programs almost $19$ points higher. This suggests that \doublecat imitates the dynamics at inference of a one-hot flow, which tracks the denoiser's trajectory, $(\pi_t^\theta(x_t))_{t\in[0,1]}$, and that the simplex noise, $Y_0^\ell$, biases sampling.

We also show that this mechanism, while being substantially cheaper computationally, outperforms plain self-conditioning, in~\Cref{tab:self-conditioning-doublecat}. Nonetheless, self-conditioning further improves \doublecat (up to $20.77\%$ at $T=0.5$); we further explore this in~\Cref{ssec:self-cond-embedded}.

\begin{takeaway}
Self-conditioning increases the density of highly likely inference regions during training, diminishing the train--test mismatch. \doublecat recovers a large part of the benefit by introducing a double-level flow, imitating self-conditioning.
\end{takeaway}

\paragraph{Discrete vs.\ continuous.} A natural question is why this discrepancy affects discrete data generation, whereas it is far less of an issue in the typical continuous case (\eg images). One possible explanation that we provide for this is the degeneracy of the target denoiser, $D_t^{\star}(x) = \E[X_1 \mid I_t = x]$. Its Jacobian is proportional to the posterior covariance,
\begin{equation}
\nabla_{x^\ell} D_t^{\star, \ell}(x) = \frac{\gamma_t}{(1-\gamma_t)^2}\cov(X_1^\ell \mid I_t = x),
\end{equation}
the trace of which is $\frac{\gamma_t}{(1-\gamma_t)^2}\left(1 - \norm{D_t^{\star, \ell}(x)}^2\right)$ for one-hot data. For (continuous) log-concave densities (\eg Gaussian), this Jacobian remains bounded by $1/\gamma_t$.

For discrete data, by contrast, the covariance does not shrink between vertices. Call a state $x$ \emph{uncommitted} at position $\ell$ if its posterior is spread over several tokens, \ie $x \in U_\eta^{t,\ell} \coloneqq \{x : 1 - \norm{D_t^{\star, \ell}(x)}^2 \geq \eta\}$ for a fixed $\eta \in (0, 1)$. On $U_\eta^{t,\ell}$, the trace of the Jacobian is at least $\eta\gamma_t/(1-\gamma_t)^2$, which diverges as $t \to 1$. Near $t \approx 1$, a small perturbation may change the target prediction drastically. Uncommitted states are also rarely seen in training: under a uniform prior, we show that the expected fraction of uncommitted positions is bounded as
\begin{equation}
    \frac{1}{L}\sum_{\ell=1}^L\Prob(I_t \in U_\eta^{t,\ell}) \leq \frac{2P_e(t)}{\eta} \xrightarrow[t \to 1]{} 0.
\end{equation}
At inference, however, $H(\hat Y_t^\ell) > 0$, and trajectories travel through this sparsely trained interior, where the model must extrapolate. A smooth extrapolation of a smooth target is nearly correct, but that of a target varying this sharply is not: any misplaced boundary flips the predicted token. These claims are detailed in~\Cref{app:doublecat-proof}.

Gaps in training coverage seem to be more benign for continuous data, but costly for discrete data, which \doublecat partially addresses by placing training mass in those regions.

\subsection{Token identifiability}
\label{ssec:token-identifiability}
\begin{figure}[tbp]
    \centering
    \includegraphics[width=\linewidth]{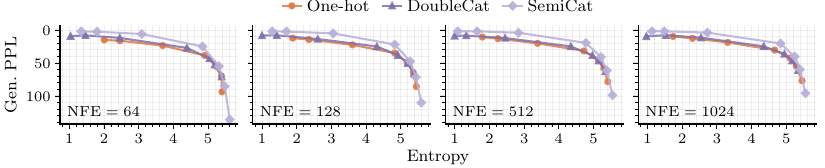}
    \caption{\genppl/entropy Pareto frontiers as measured by GPT-2 for one-hot flows trained on OpenWebText with sequence lengths of $1024$ tokens, using the GPT-2 tokeniser.}
    \label{fig:cat-pareto-grid}
\end{figure}
Despite the improvement offered by \doublecat, the model still heavily suffers from name errors ($30.02\%$). We attribute this to \emph{token identifiability}~\citep{candi}: when referencing another variable, the model struggles to identify current tokens because of the large noise in $x_t$, and therefore cannot condition on it precisely: the Gaussian noise has an expected norm $\sqrt{\vocabsize}$, but the signal has a norm of $1$.

To mitigate this issue, we add a \emph{revealing} mechanism akin to the one in CANDI: at random times $\tau^\ell$, a token is revealed as a one-hot input. This way, the model can cleanly condition on revealed inputs. The resulting method is \semicat.

Formally, the \semicat interpolants are given by:
\begin{equation}
    I_t^{\mathsf{sc},\ell}(X_0, X_1) = I_t^{\mathsf{dc},\ell}(X_0, X_1)\indic(t < \tau^\ell) + \indic(t \geq \tau^\ell)X_1^\ell,
\end{equation}
where $\indic$ is the indicator function. An optimal model carries over tokens for which $t \geq \tau^\ell$, and so we parametrise it as such~\citep{mdlm}. The training therefore reduces to the following cross-entropy loss:
\begin{equation}
    \label{eq:semicat-ce}
    \loss(\theta) = -\E_{X_0, X_1, t, \boldsymbol{\tau}}\sum_{\ell=1}^L \indic(t < \tau^\ell) \log \left\langle X_1^\ell, \pi^{\theta,\ell}_t(I_t^{\mathsf{sc}})\right\rangle,
\end{equation}
where $\boldsymbol{\tau} = (\tau^1, \ldots, \tau^L)$. Following \doublecat, we ignore the $(Y_t)_{t\in[0,1]}$ terms at inference time, thus amounting to CANDI-like sampling: each still-continuous position is revealed between times $t_n$ and $t_{n+1}$ with probability $\frac{t_{n+1}-t_n}{1-t_n}$, committing a token drawn from $\pi^\theta$, while the remaining positions follow the one-hot flow, as summarised in~\Cref{alg:semicat-sampling}.

In~\Cref{tab:tinygsm-doublecat}, \semicat produces only $10.39\%$ runs with name errors and programs that run $89.23\%$ of the time, \emph{and reaches the discrete diffusion level with an accuracy of $\mathbf{32.90\%}$ without self-conditioning}, on TinyGSM. We note that revealing inputs \`a la CANDI is not sufficient on its own.
On OWT, \semicat also dominates the \genppl/Entropy Pareto frontier of one-hot flows at all NFEs (\Cref{fig:cat-pareto-grid}).

For better comparison, we also train \semicat with more standard hyperparameters (including batch size $512$ and $250$k training steps, see~\Cref{app:hyperparameters} and~\Cref{tab:hyperparameters-tinygsm-standard}). We reach $\mathbf{38.67\%}$ with a naive sampler. Adaptive noise allocation~\citep[InfoNoise,][]{raya2026noiseschedulinginformationguidedallocation} brings our performance further up to $\mathbf{41.09\%}$, and a better tuned confidence-based sampler allows \semicat to attain $\mathbf{49.94\%}$.

\begin{takeaway}
Large Gaussian noise makes the tokens hard to identify from the current state, $x_t$. Introducing reveals like CANDI considerably helps in this regard.
\end{takeaway}

While \semicat obtains high performance, its stochastic reveals and piecewise continuous trajectories make it a harder candidate to accelerate using typical \emph{continuous} techniques. We discuss in~\Cref{app:asynchronous-schedules} the possibility of designing asynchronous schedules (\ie schedules per token position, $\ell$) to alleviate this issue. Besides, we briefly explore more general hybrid corruption processes in~\Cref{app:hybrid-methods}, and show that \semicat can obtain competitive performance compared to FMLM+ when distilling our model using D-MMD~\citep{dmmd,mmd} in~\Cref{app:dmmd}, in the few-step regime.

\section{Designing embeddings}
\label{sec:embeddings}
One-hot flows are costly, especially given that modern LM vocabulary sizes can exceed $200$k tokens, and incur a large memory and computational penalty, requiring the materialisation and manipulation of an $L\times\vocabsize$ tensor per data point in a batch. It is therefore natural to depart from this representation and work on smaller embedded spaces. We show that their DER can become problematic with pure cross-entropy training, and that self-conditioning aids in token identifiability.

\subsection{Your one-hot flow is secretly embedded}
\label{ssec:one-hot-is-embed}
The open-sourced one-hot flow works for language modelling~\citep{cfm,fmlm} use the adapted DiT architecture~\citep{dit} from~\citet{sedd}, with a simple linear projection for the input, which we denote by $W \in \R^{\vocabsize \times d}$. This means that the first layer of the architecture after the projection receives an input distributed as
\begin{equation}
    \label{eq:one-hot-is-embed}
    \forall \ell \in \intrange{1}{L}, \qquad (I_tW)_\ell \sim \normal\left(\gamma_t (X^\ell_1)^\top W, (1-\gamma_t)^2 W^\top W\right).
\end{equation}
This effectively shows that one-hot flows operate on an embedded space defined by their projection matrix, but which also modifies the injected noise, giving it a covariance structure defined by $W^\top W$. Note that~\citet{candi} make a similar observation and use it to accelerate inference. This naturally holds for \doublecat and \semicat too. We discuss how to \emph{train} \semicat directly on the embedded space in~\Cref{app:doublecat-embed-reprod}, and attain the same end performance but with a more unstable training due to a Cholesky decomposition.

\subsection{Your embedded flow is secretly one-hot}
\label{ssec:embed-is-one-hot}
\begin{table}[t]
\centering
\small
\setlength{\tabcolsep}{4pt}
\begin{tabularx}{\textwidth}{@{}Xcccccc@{}}
\toprule
\multirow{2}{*}[-2pt]{\textbf{Method}} & \multicolumn{4}{c}{\textbf{Accuracy} (\%, $\uparrow$)} & \multicolumn{2}{c}{\textbf{Best} (\%)} \\
\cmidrule(lr){2-5}\cmidrule(lr){6-7}
& $T = 1$ & $T = 0.5$ & $T = 0.2$ & $\mathrm{top}\text{-}k\text{=}1$ & Runs ($\uparrow$) & Name err. ($\downarrow$) \\
\midrule
\emph{Ambient sphere}\\
\quad $d = 16$ & $6.97$ & $9.40$ & $8.95$ & $8.87$ & $58.83$ & $40.71$ \\
\quad $d = 128$ & $7.88$ & $9.25$ & $9.63$ & $10.24$ & $56.48$ & $42.76$ \\
\quad $d = 768$ & $\mathbf{8.64}$ & $\mathbf{12.89}$ & $\mathbf{13.12}$ & $\mathbf{12.28}$ & $\mathbf{66.72}$ & $\mathbf{32.75}$ \\
\midrule
\emph{Plaid, $d = 16$}\\
\quad self-cond. & $\mathbf{19.71}$ & $\mathbf{28.43}$ & $11.30$ & $2.81$ & $\mathbf{84.38}$ & $9.17$ \\
\quad without self-cond. & $8.19$ & $10.84$ & $11.98$ & $\mathbf{11.14}$ & $72.25$ & $26.91$ \\
\emph{Plaid, $d = 128$}\\
\quad self-cond. & $17.89$ & $22.44$ & $\mathbf{16.22}$ & $10.46$ & $82.94$ & $12.66$ \\
\quad without self-cond. & $6.75$ & $9.55$ & $9.48$ & $9.78$ & $62.77$ & $36.54$ \\
\emph{Plaid, $d = 768$}\\
\quad self-cond. & $17.97$ & $24.56$ & $12.96$ & $5.00$ & $78.32$ & $\mathbf{3.56}$ \\
\quad without self-cond. & $4.62$ & $5.23$ & $8.79$ & $8.04$ & $57.70$ & $41.77$ \\
\midrule
\hlrow \emph{Plaid variants, $d = 16$, without self-cond.} & & & & & & \\
\hlrow \quad with posterior \doublecat & $12.89$ & $15.85$ & $13.87$ & $14.18$ & $74.07$ & $25.32$ \\
\hlrow \quad stochastic codebook & $9.55$ & $16.15$ & $18.35$ & $18.42$ & $68.31$ & $31.08$ \\
\hlrow \quad with posterior as self-cond. & $10.84$ & $16.53$ & $17.97$ & $18.35$ & $67.10$ & $32.07$  \\
\hlrow \qquad self-cond. at inference & $\mathbf{16.60}$ & $\mathbf{18.04}$ & $\mathbf{19.86}$ & $\mathbf{20.24}$ & $\mathbf{76.12}$ & $\mathbf{23.12}$ \\
\bottomrule
\end{tabularx}
\caption{TinyGSM results at $256$ NFE. Plaid learns its log-SNR endpoints; the variants all pin them. ``Without self-cond.'' refers to checkpoints trained without self-conditioning. $T$ denotes the temperature applied to $\pi^\theta$. The models are trained with batch size $256$ for $200$k steps. See~\Cref{app:hyperparameters} for more details.}
\label{tab:tinygsm-embedded}
\end{table}

Considering embedded methods, we know that \doublecat and \semicat on the embedded space from~\Cref{ssec:one-hot-is-embed} obtain high performance, and Plaid~\citep{plaid} does too. In contrast, naive embedded training underperforms: denoted as ``ambient sphere'', we train (with default cross-entropy) a simple flow interpolating between a Gaussian and the final unit-normalised embedding, \ie $\norm{\emb(v)} = 1$ for any $v\in V$. Mathematically, we write
\begin{equation}
    I_t^\ell(X_0, X_1) = (1-t)X_0^\ell + t \emb(X_1^\ell).
\end{equation}
In~\Cref{tab:tinygsm-embedded}, the ambient sphere underperforms Plaid with self-conditioning; besides, we note that the ambient sphere also largely underperforms on OWT, with a very high \genppl (beyond $200$) at matched data entropy.

The axis that helps us understand these results is the training objective. The simple cross entropy objective of the ambient sphere is pathological in two ways: \textit{(i)} The model only updates the embeddings of the tokens involved in the batch, thus not updating the entire geometry of $\embspace$. Plaid updates all embeddings through its MSE diffusion term, and \semicat and \doublecat do so through the covariance structure of their embedded noise (\Cref{ssec:one-hot-is-embed}). \textit{(ii)} Cross-entropy ignores the geometry of the embeddings, and spreads tokens apart on the sphere approximately recovering the one-hot geometry in the DER sense~\citep{candi,duo}, \ie $P_e(t)$ tends towards that of the one-hot case. To see this, observe that the DER for embeddings can be generalised to
\begin{equation}
\label{eq:geom-edr}
    P_e(t) \coloneqq \frac{1}{L} \sum_{\ell=1}^L \Prob(\argmin_v \norm{ I_t^\ell - \emb(v) }^2 \neq \argmax~X_1^\ell).
\end{equation}
For $\emb:v\mapsto e_v$, this recovers the usual case of~\eqref{eq:base-edr}. \Cref{prop:emb-dep-sched} exhibits that the DER~\eqref{eq:geom-edr} depends solely on the interactions $\emb(j) - \emb(i)$ and $\langle\emb(i),\emb(j)\rangle$ for any $i, j \in V$. In~\Cref{app:ce-mse-geometry}, we show informally that, under a Bayes-optimal denoiser, the cross-entropy pushes only the true token's embedding away from the posterior-mean embedding (\ie away from the tokens it is confused with). Plaid's MSE term contains an opposite, attracting component, which on its own would collapse embeddings to zero. Plaid thus balances the two, but pure cross-entropy training only repels, and we recover the one-hot DER regime. Indeed,~\Cref{app:one-hot-limit} shows that unit-norm embeddings maximally spread in dimension $d$ have inner product terms that tend to $0$ as $d$ increases, and in turn the DER tends to the one-hot case.

\begin{figure}[t]
\centering
\begin{subfigure}[t]{0.48\textwidth}
\centering
\includegraphics{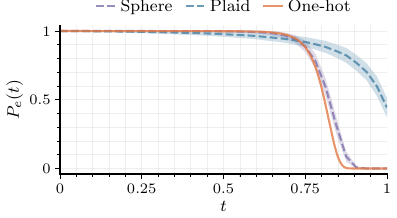}
\caption{Decoding error rates, $P_e$.}
\label{fig:edr-embedded-sphere}
\end{subfigure}\hfill
\begin{subfigure}[t]{0.48\textwidth}
\centering
\includegraphics{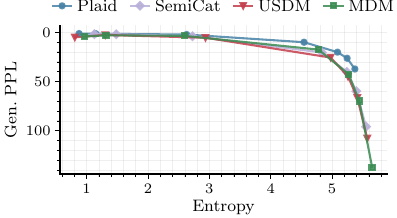}
\caption{\genppl/Entropy Pareto fronts on OWT.}
\label{fig:cross-method-pareto-owt-1024}
\end{subfigure}
\caption{DER of embedded geometries, and the quality/diversity trade-off across methods on OWT. \textbf{(Left)} The DER of a trained sphere model on OWT and that of Plaid both with $d=768$, against the known one-hot DER. \textbf{(Right)} \genppl/Entropy Pareto frontiers on OWT, with Plaid for $d=16$.}
\label{fig:sphere-decode-cross-pareto}
\end{figure}
Empirically, in~\Cref{fig:edr-embedded-sphere}, we measure through Monte Carlo estimates the DER, $P_e$, across time for a trained ambient sphere and a Plaid model ($d=768$) on OWT. The measured and one-hot DER curves are indistinguishable except for a small gap at $t \approx 0.85$. Remark that Plaid's embeddings never induce $P_e(t) = 0$ even at $t=1$ due to the learnt VDM schedule endpoints (see~\Cref{app:vdm-snr}). At $d=768$, the ambient sphere's spreading of tokens makes them more identifiable, yielding lower name errors and higher run rate when compared with Plaid without self-conditioning.

\begin{takeaway}
Cross-entropy training spreads tokens apart, inducing the same DER as a one-hot flow. This may explain why prior work requires more advanced schedules~\citep{hypersphere,elf,langflow}; an alternative is updating the geometry of all embeddings jointly.
\end{takeaway}

\subsection{Self-conditioning on embedded spaces}
\label{ssec:self-cond-embedded}
In~\Cref{tab:tinygsm-embedded}, the performance of the Plaid runs without self-conditioning \emph{monotonically worsens} with $d$. This substantiates our findings: in higher dimensions, the discrepancy between the train- and inference- time statistics, $I_t^\ell$ and $x_t^\ell$, is further widened. For this reason, we turn to adapting \doublecat.

\paragraph{Embedded DoubleCat.} Naively porting \doublecat to Plaid, by simply tracking $Y_0^\ell = \vv^\top W$ until the true $Y_1^\ell = \emb(X_1^\ell)$, for $\vv \sim \unif(\Delta^\vocabsize)$, breaks. This happens because the model is given an input noisier than expected by the statistics of the ELBO, and applying the large weight $\SNR'$ to a non-zero loss makes the model diverge. To solve this, we attempted two techniques that provide more stability. Note that we freeze the endpoints of the learnt VDM schedule for additional stability.

The first one avoids having non-zero terms near $t\approx 1$: we instead flow the model towards its token-wise Bayesian posterior expectation at that point, $I_t^\ell = x_t^\ell$, and time $t$, assuming a uniform prior over the tokens and unit-normalised embeddings. That is to say, we set
\begin{equation}
\label{eq:bayes-posterior}
    D_t^\ell \coloneqq \E\left[\emb(X_1^\ell) \mid I_t^\ell = x_t^\ell\right] = \softmax_{v\in V}\left(\frac{\beta_t\left\langle \emb(v), x_t^\ell\right\rangle}{\alpha_t^2}\right)W,
\end{equation}
and construct the input instead as $\tilde{I}_t^\ell = \alpha_t X_0^\ell + \beta_t D_t^\ell$. Note that we recover similar statistics, since $Y_0^\ell$ is deterministically $\pcatunif^\top W$ and $Y_1^\ell = \emb({X_1^\ell})$. As shown in~\Cref{tab:tinygsm-embedded}, this modestly improves over the base $11.98\%$, reaching $15.85\%$.

The second one is denoted by ``stochastic codebook'' in~\Cref{tab:tinygsm-embedded}, where we defined embeddings to be Gaussians, regularised to $1/d$ variance per coordinate and unit-norm mean. The goal is to increase the density of under-trained regions by this added stochasticity. This is most effective at low temperatures, but does not recover the self-conditioning performance. This resembles the noising tricks of~\citet{elf,ldlm}.

\paragraph{Self-conditioning and token identifiability.} We hypothesise that self-conditioning \emph{also} enhances token identifiability. To show that, instead of removing self-conditioning outright, we attempt to replace it with a cheaper proxy at train time, which is rendered possible with small $d$. While the token identifiability issue is less severe in small dimension, the mismatch in norms between the Gaussian and the embedding of norm \emph{one} is still present. In contrast, the self-conditioning signal,
\begin{equation}
    \SC = \embhead(x_t, \vzero, t) \equiv \E_{X_1\sim \pi^\theta_t(\cdot, \vzero, t)}\left[\emb(X_1)\mid I_t = x_t \right],
\end{equation}
lies exactly on the convex hull formed by the embeddings. Instead of self-conditioning, we feed the model the Bayesian posterior expectation (\ie \eqref{eq:bayes-posterior}), $\SC = D_t$, at train time. At inference, we try both the posterior ($\SC = D_t$, where $D_t = (D_t^1, \ldots, D_t^L)$) and true self-conditioning ($\SC = \pi^\theta_{t_{k-1}}(x_{t_{k-1}})$). Since $D_t$ is deterministic given $I_t = x_t$, it adds no information to the input, and gains must come from a better representation of the state. This approach attains $18.35\%$ (vs.\ $11.98\%$ without self-conditioning), and $20.24\%$ with true self-conditioning at inference---at no additional forward pass cost at train-time.

We support this claim with the following idealised analysis, proven in~\Cref{app:self-cond-identifiability}: we assume a perfect denoiser, self-conditioning computed from $I_t$ as at train-time, and unit-norm embeddings; we use the Plaid path, $I_t^\ell(X_0, X_1) = \alpha_t X_0^\ell + \beta_t \emb(X_1^\ell)$. Let $R$ be a function of $I_t$, \ie $R = f(I_t) \in \R^d$. Define $\distortion_B(R)$ to be the largest possible distortion between $g(\emb(X_1^\ell))$ and $g(R)$ over all $B$-Lipschitz functions $g:\R^d\to\R^d$:
\begin{equation}
    \label{eq:distortion-def}
    \distortion_B(R) \coloneqq \sup_{\mathrm{Lip}(g) \leq B} \E\norm{g(R) - g(\emb(X_1^\ell))}^2.
\end{equation}
This definition quantifies the worst-case error gap for Lipschitz maps evaluated on representations of $I_t$. Since the $\sup$ is attained by $g:x\mapsto Bx$, we have $\distortion_B(R) = B^2\E\norm{R-\emb(X_1^\ell)}^2$.
Define the posterior mean embedding as $D_t^{\star, \ell} = \E[\emb(X_1^\ell) \mid I_t]$. For any $R$, we know that
\begin{equation}
    \distortion_B\left(D_t^{\star, \ell}\right) \leq \distortion_B\left(R\right).
\end{equation}
That is to say, $D_t^{\star, \ell}$ is optimal. Furthermore, we compare $I_t^\ell$ (normalised by $\beta_t$ to account for the norm mismatch) and $D_t^{\star,\ell}$, and prove that
\begin{equation}
    \distortion_B(I_t^\ell/\beta_t) - \distortion_B(D_t^{\star, \ell}) = B^2\left[\frac{d}{\SNR(t)} - \mmse_t^\ell\right],
\end{equation}
where $\mmse_t^\ell$ is the minimal MSE at time $t$ and position $\ell$. (\Cref{app:self-cond-identifiability} treats the more general case of an imperfect denoiser.) We note that the best attainable $\mmse$ over codebooks of dimension $d$ is non-increasing in $d$, while the distortion gap grows at least linearly in $d$. This mechanism also provides another possible explanation for the growing discrepancies of Plaid with $d$ in~\Cref{tab:tinygsm-embedded}, but it does not cover the aforementioned train/inference mismatch nor inference-time self-conditioning.

\begin{takeaway}
Self-conditioning holds a more identifiable representation of the state, on the convex hull formed by the embeddings, which aids in token identifiability.
\end{takeaway}

Cross-method Pareto frontiers are available in~\Cref{fig:cross-pareto-front-owt}, and for $1024$ NFE in~\Cref{fig:cross-method-pareto-owt-1024}.

%% file: section/experiment.tex
\section{Scaling SemiCat}
\begin{figure}[tbp]
    \centering
    \includegraphics[width=\linewidth]{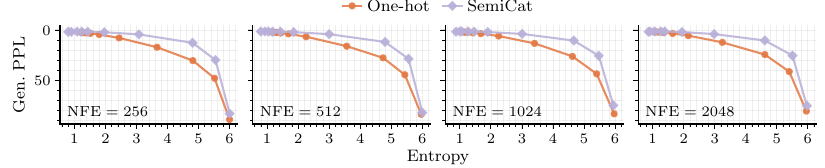}
    \caption{\genppl as measured by Qwen $2.5$ $7$B for models trained on Nemotron with sequence lengths of $2048$ tokens, using the TikToken tokeniser.}
    \label{fig:big-run-pareto-gen}
\end{figure}
\label{sec:large-experiment}
\begin{table}[t]
\centering
\vskip 0.1in
\begin{small}
\begin{tabularx}{\textwidth}{@{}Xcccccc@{}}
\toprule
\textbf{Method} & \textbf{ARC-E} & \textbf{BoolQ} & \textbf{OBQA} & \textbf{PIQA} & \textbf{RACE} & \textbf{SIQA} \\
\midrule
\textit{Chance} & $24.7$ & $50.4$ & $26.6$ & $51.6$ & $24.2$ & $32.2$ \\
\midrule
LLaDA-8B-Base~\citep{llada} & -- & -- & -- & $74.4$ & -- & -- \\
AR$^\dagger$ {\small\textit{(autoregressive)}} & $72.7$ & $71.9$ & $40.4$ & $78.1$ & $36.2$ & $41.9$ \\
AR$^\star$ {\small\textit{(autoregressive)}} & $73.1$ & $63.6$ & $40.2$ & $76.9$ & $36.3$ & $43.5$ \\
\cmidrule{1-7}
Duo$^\dagger$ {\small\textit{(uniform diffusion)}} & $53.4$ & $59.6$ & $\mathbf{33.0}$ & $62.7$ & $35.0$ & $39.0$ \\
MDLM$^\dagger$ {\small\textit{(masked diffusion)}} & $50.5$ & $62.8$ & $32.0$ & $62.2$ & $34.7$ & $39.2$ \\
\midrule
\hlrow One-hot flow & $42.6$ & $37.9$ & $26.8$ & $56.4$ & $28.2$ & $36.4$ \\
\hlrow \semicat~{\small\textit{(ours)}} & $\mathbf{61.8}$ & $\mathbf{64.8}$ & $30.2$ & $\mathbf{70.3}$ & $\mathbf{35.3}$ & $\mathbf{40.4}$ \\
\bottomrule
\end{tabularx}
\end{small}
\caption{Zero-shot accuracy (\%) on likelihood-based benchmarks. $^\star$Models from~\citet{scalingcfm}. $^\dagger$Models from~\citet{scalingbeyondmaskedlm}. Both with $1.7$B parameters, trained on Nemotron for $2.1$T tokens. \semicat and one-hot flow likelihoods are estimated by the cross-entropy loss of the continuation (see~\cref{eq:one-hot-flow-ce,eq:semicat-ce}). Best non-AR results at $1.7$B parameters per column are in \textbf{bold}.}
\label{tab:likelihood-benchmarks}
\vspace{-1.25em}
\end{table}
We train a larger \semicat model following the protocol of~\citet{scalingcfm,scalingbeyondmaskedlm} closely: the architecture is the same DiT~\citep{dit,sedd} with $1.7$B parameters, trained on the Nemotron Pretraining dataset~\citep{nemotron} for $2.1$T tokens using the TikToken tokeniser~\citep{tiktoken}, with sequences of length $2048$. We use $256\times$ H100 GPUs for the pre-training run with a batch size of $10$ per GPU. We also train a one-hot flow baseline on the same settings~\citep{scalingcfm}. See~\Cref{app:hyperparameters} for more details.

\paragraph{Likelihood-based benchmarks.} We evaluate our model over the same multiple-choice question answering (MCQA) benchmarks as in~\citet{scalingcfm,scalingbeyondmaskedlm}: ARC-easy~\citep{arceasy}, BoolQ~\citep{boolq}, OBQA~\citep{obqa}, PIQA~\citep{piqa}, RACE~\citep{RACE} and SIQA~\citep{siqa}. Instead of using an ELBO, we find the plain cross-entropy loss to be a better proxy for the likelihood for both the one-hot flow and \semicat.

The results are available in~\Cref{tab:likelihood-benchmarks}. \semicat outperforms our one-hot flow baseline on every MCQA task, and \emph{outperforms the discrete diffusion baselines on ARC-E by $8.4$ points, on BoolQ by $2.0$, PIQA by $7.6$, and on SIQA by $1.2$}. On OBQA, our results are behind by $2.8$ points, and we slightly exceed Duo and MDLM on RACE ($+0.3$).

\paragraph{Unconditional generation.} Following the setting of~\citet{scalingcfm}, \semicat outperforms one-hot flows on unconditional generation, as measured by the \genppl/Entropy Pareto frontiers provided in~\Cref{fig:big-run-pareto-gen}. We generate sequences of $2048$ tokens, and sweep over NFEs and temperature. The \genppl is computed using Qwen $2.5$ $7$B~\citep{qwen2}. At all NFEs, \semicat dominates the one-hot flow baseline.

%% file: section/related.tex
\section{Related Work}

\paragraph{Continuous diffusion for language modelling.} Gaussian interpolation on one-hot encodings was introduced from a variational perspective in VFM~\citep{vfm}, and has been shown to be amenable to self-distillation and steering by CFM/FMLM/DFM~\citep{cfm,fmlm,dfm}. \citet{fmlmstar} analyse self-conditioning through a fixed-point iteration lens;~\citet{discreteDataFail} do so through that of multi-modality.

\paragraph{Embedded space design.} Since those advances, \sflm and von Mises-Fisher flows~\citep{hypersphere,vmfFlow} were the first to show reasonable performance on TinyGSM for continuous methods, by flowing on the sphere but by taking different probability path constructions. In parallel, geometrically simpler approaches have been developed. LangFlow~\citep{langflow} use a simple linear projection, and ELF~\citep{elf} attempt learnt embeddings but have found using a pre-trained encoder to be much more effective.~\citet{ldlm} show how to leverage a pretrained AR model to obtain strong embeddings.

We provide additional references in~\Cref{app:related-work}.

%% file: section/conclusion.tex
\section{Conclusion}
In this work, we show that self-conditioning aids in two important aspects of CDLMs: trajectory ambiguity (increasing the density of inference-like points at train time) and token identifiability (the ability to identify the expected endpoint of the trajectory given its current state). For the former, \doublecat can recover a large part of those benefits; for the latter, providing a better representation of the state, such as feeding the analytically available Bayesian posterior, helps substantially at a lower cost. We show that one-hot methods in fact evolve over an embedded space with anisotropic noise, while embedded methods trained naively recover one-hot properties in terms of their DER. \semicat, a hybrid method, can attain discrete-diffusion-level performance on our benchmarks. At scale, \semicat outperforms discrete diffusion on four out of six of our likelihood tasks, and surpasses one-hot flows on unconditional generation. We leave to future work accelerating these methods with the continuous toolbox, and understanding how that toolbox interacts with our design choices.

%% file: section/appendix.tex
\clearpage
\section{Extended related work}
\label{app:related-work}
\paragraph{Earlier approaches.} Earlier approaches of continuous diffusion for language modelling include CDCD~\citep{cdcd}, Diffusion-LM~\citep{diffusionlm}, Bayesian Flow Networks~\citep{bfngraves}.

\paragraph{Geometric constructions.} Earlier geometric work includes~\citet{fisher,categoricalflowmatching}, which flowed both on Fisher-Rao simplex and on the positive orthant on the sphere, but that failed to scale to larger vocabulary sizes.

\paragraph{Beyond token-wise embeddings.} Latent Thought Flow~\citep{ltf2026latent} explores the idea of compressing the sequence length. We believe that this general direction is also very promising, but the exploration thereof should be deferred to another dedicated work. Other works consider using an existing strong AR model for encoding, such as ~\citet{ldlm}, which uses the contextual embeddings of GPT-2~\citep{gpt2}. While offering better performance, the approach is naturally more expensive computationally at train time.

\paragraph{Hybrid methods.} Besides CANDI~\citep{candi} and FMLM+~\citep{fmlmplus}, we note also the existence of the literature of \emph{discrete} diffusion with continuous augmentations, for instance via latents~\citep{cadd,ccdd}. Concurrently, \citet{jutrasdubé2026stickyjumpdiffusionsunifying} formalised these hybrids, together with masked and continuous diffusion, as special cases of a single process on the embedding space. Tokens form discrete anchors that the reverse process can jump onto, and the forward process determines when and where those jumps happen. In their framework, the CANDI-style hybrid underlying \semicat is recovered as the identity-kernel case. We see this as a principled way to design hybrid corruption processes, complementary to the variants we explore in~\Cref{app:hybrid-methods}.

\section{DoubleCat derivations}
\label{app:doublecat-proof}
\paragraph{Interpolant form of the inference trajectory.} A similar result can also be found in~\citet{fmlm,dfm}, where the authors provide an integral form of the two-time denoiser. Fix a position $\ell$ and write $\pi_t = \pi_t^{\theta,\ell}(\hat I_t)$ for the prediction of any denoiser along the inference trajectory, which solves
\begin{equation}
    \frac{\mathrm{d} \hat I_t^\ell}{\mathrm{d} t} = \dot\gamma_t\frac{\pi_t - \hat I_t^\ell}{1-\gamma_t}, \qquad \hat I_0^\ell = X_0^\ell.
\end{equation}
Then $\frac{\mathrm{d}}{\mathrm{d} t}\left[\hat I_t^\ell/(1-\gamma_t)\right] = \dot\gamma_t\pi_t/(1-\gamma_t)^2$. Integrating with $\gamma_0 = 0$, and using $\int_0^t \dot\gamma_s/(1-\gamma_s)^2\mathrm{d} s = \gamma_t/(1-\gamma_t)$, we obtain, for $t \in (0, 1)$,
\begin{equation}
    \hat I_t^\ell = (1-\gamma_t)X_0^\ell + \gamma_t\hat Y_t^\ell, \qquad \hat Y_t^\ell \coloneqq \int_0^t w_t(s)\pi_s\mathrm{d} s, \qquad w_t(s) \coloneqq \frac{1-\gamma_t}{\gamma_t}\frac{\dot\gamma_s}{(1-\gamma_s)^2}.
\end{equation}
Since $w_t \geq 0$ integrates to one on $[0, t]$, $\hat Y_t^\ell$ is a convex combination of the predictions, and thus lies on the simplex. Finally, by concavity of the entropy and Jensen's inequality, $H(\hat Y_t^\ell) \geq \int_0^t w_t(s)H(\pi_s)\mathrm{d}s$.

\paragraph{Degeneracy of the target denoiser.} Write $a_t \coloneqq \gamma_t/(1-\gamma_t)^2$ and $\sigma_t \coloneqq (1-\gamma_t)/\gamma_t$, so that $a_t\sigma_t^2 = 1/\gamma_t$. Since $I_t^\ell(X_0, X_1) = (1-\gamma_t)X_0^\ell + \gamma_t X_1^\ell$, the posterior reads
\begin{equation}
    p(x_1 \mid I_t = x) \propto p(x_1)\exp\left(-\frac{\norm{x - \gamma_t x_1}^2}{2(1-\gamma_t)^2}\right) \propto \tilde p(x_1)\exp\left(a_t\langle x, x_1\rangle\right),
\end{equation}
where $\tilde p(x_1) \coloneqq p(x_1)\exp(-\gamma_t^2\norm{x_1}^2/(2(1-\gamma_t)^2))$ does not depend on $x$. Hence, $\nabla_x \log p(x_1 \mid I_t = x) = a_t(x_1 - D_t^\star(x))$, and differentiating $D_t^\star(x) = \E[X_1 \mid I_t = x]$ under the expectation gives
\begin{equation}
    \nabla D_t^\star(x) = a_t\E\left[X_1(X_1 - D_t^\star(x))^\top \mid I_t = x\right] = a_t\cov(X_1 \mid I_t = x).
\end{equation}
This holds for any prior $p$. Consider the $(\ell, \ell)$ block, $\nabla_{x^\ell} D_t^{\star,\ell}(x) = a_t\cov(X_1^\ell \mid I_t = x)$. For one-hot data, $X_1^\ell \in \{e_v\}_{v\in V}$ and $D_t^{\star,\ell}(x) \eqqcolon q$ is the posterior over tokens, so that $\cov(X_1^\ell \mid I_t = x) = \mathrm{diag}(q) - qq^\top$, the trace of which is $1 - \norm{q}^2$. In particular, on $U_\eta^{t,\ell}$, $\mathrm{tr}\nabla_{x^\ell} D_t^{\star,\ell}(x) \geq \eta a_t$, which blows up as $t$ tends to $1$.

In contrast, for continuous densities, if $p_1$ is log-concave, the posterior is $p_1$ times a Gaussian of covariance $\sigma_t^2\matid$, and the Brascamp--Lieb variance inequality~\citep{brascampLieb} bounds the variance of any linear function $x\mapsto \langle u, x\rangle$ by $\sigma_t^2\norm{u}^2$, \ie $\cov(X_1 \mid I_t = x) \preceq \sigma_t^2\matid$. Therefore, for all $x$ and $t$,
\begin{equation}
    0 \preceq \nabla D_t^\star(x) \preceq a_t\sigma_t^2\matid = \frac{1}{\gamma_t}\matid,
\end{equation}
and so the Jacobian remains bounded as $t$ tends to $1$, since $\gamma_1 = 1$.

Finally, for the training mass, consider the token-wise posterior under a uniform prior, $q = q(I_t^\ell)$ (which coincides with $D_t^{\star, \ell}$ for this choice of prior), as in~\eqref{eq:bayes-posterior}, and let $q_{\max} \coloneqq \max_v q_v$. Since $q_{\max}^2 \leq \norm{q}^2 \leq q_{\max}$,
\begin{equation}
    1 - q_{\max} \leq 1 - \norm{q}^2 \leq (1 - q_{\max})(1 + q_{\max}) \leq 2(1 - q_{\max}).
\end{equation}
Under a uniform prior, $\argmax_v~q_v = \argmax~I_t^\ell$, so that $\E[1 - q_{\max}] = \Prob(\argmax~I_t^\ell \neq \argmax~X_1^\ell) \eqqcolon P_e^\ell(t)$. Hence, $P_e^\ell(t) \leq \E[1 - \norm{q(I_t^\ell)}^2] \leq 2P_e^\ell(t)$, and Markov's inequality gives $\Prob(I_t \in U_\eta^{t,\ell}) \leq 2P_e^\ell(t)/\eta$, which averages over positions to $2P_e(t)/\eta$.

\section{Extended SemiCat presentation}
\label{app:semicat}
The sampling algorithm for \semicat is provided in~\Cref{alg:semicat-sampling}. It follows the sampling procedure of CANDI~\citep{candi}. Moreover, we propose~\Cref{alg:semicat-tuned-sampling} for a more sophisticated sampling scheme, which mostly introduces DDIM-style~\citep{ddim} renoising, and confidence-based ``remasking'' (return to continuous state) inspired by the procedure of ReMDM~\citep{remdm}.
\begin{algorithm}[h]
\caption{\semicat sampling.}
\label{alg:semicat-sampling}
\begin{algorithmic}[1]
\Require Denoiser $\pi^\theta$, schedule $\gamma$, time grid $0 = t_0 < \dots < t_N = 1$, temperature $T$.
\State $x^\ell \sim p_0$ and $r^\ell \gets 0$ for all $\ell$ \Comment{$r^\ell = 1$: revealed}
\For{$n = 0, \ldots, N-1$}
    \State $p \gets \pi^\theta_{t_n}(x)$
    \For{$\ell$ such that $r^\ell = 0$}
        \State $x^\ell \gets x^\ell + \frac{\gamma_{t_{n+1}} - \gamma_{t_n}}{1 - \gamma_{t_n}}\left(p^\ell - x^\ell\right)$ \Comment{Euler step towards $\pi^\theta$}
        \State Draw $M^\ell \sim \ber\left(\frac{t_{n+1} - t_n}{1 - t_n}\right)$
        \If{$M^\ell = 1$}
            \State Draw $v^\ell \sim \cat\big((p^\ell)^{1/T}\big)$
            \State $x^\ell \gets e_{v^\ell}$
            \State $r^\ell \gets 1$
        \EndIf
    \EndFor
\EndFor
\State \Return $\left(\argmax_{v} x^{\ell}_v\right)_{\ell=1}^L$
\end{algorithmic}
\end{algorithm}

\begin{algorithm}[h]
\caption{\semicat tuned sampling.}
\label{alg:semicat-tuned-sampling}
\begin{algorithmic}[1]
\Require Denoiser $\pi^\theta$, schedule $\gamma$, uniform time grid $t_n = n/N$, DDIM-stochasticity $\eta_\mathrm{max}$, remask budget $\rho$, uniform share $\lambda$.
\State $x^\ell \sim p_0$ and $r^\ell \gets 0$ for all $\ell$ \Comment{$r^\ell = 1$: revealed}
\For{$n = 0, \ldots, N-1$}
    \For{$\ell$ such that $r^\ell = 1$}
        \State Draw $R^\ell \sim \ber\left(\min\!\big(1, \frac{\rho}{N}\big(\lambda + (1 - \lambda)(1 - c^\ell)\big)\big)\right)$\Comment{Confidence-based remask}
        \If{$R^\ell = 1$}
            \State Draw $\epsilon^\ell \sim p_0$
            \State $x^\ell \gets (1 - \gamma_{t_n})\epsilon^\ell + \gamma_{t_n} x^\ell$
            \State $r^\ell \gets 0$
        \EndIf
    \EndFor
    \State $p \gets \pi^\theta_{t_n}(x)$
    \State $q \gets 1 - \frac{L\,(1 - t_{n+1})}{\lvert\{\ell : r^\ell = 0\}\rvert}$ if $n < N - 1$, else $q \gets 1$ \Comment{Paced reveal}
    \State $q\gets \mathrm{clamp}(q, 0, 1)$
    \State $\eta \gets \eta_\mathrm{max} \sqrt{1 - \Big(\frac{\gamma_{t_n}(1 - \gamma_{t_{n+1}})}{\gamma_{t_{n+1}}(1 - \gamma_{t_n})}\Big)^2}$
    \For{$\ell$ such that $r^\ell = 0$}
        \State $\hat x_0^\ell \gets \frac{x^\ell - \gamma_{t_n} p^\ell}{1 - \gamma_{t_n}}$, draw $\epsilon^\ell \sim p_0$, set $\hat x_0^\ell \gets \sqrt{1 - \eta^2}\,\hat x_0^\ell + \eta\epsilon^\ell$
        \State $x^\ell \gets (1 - \gamma_{t_{n+1}})\hat x_0^\ell + \gamma_{t_{n+1}} p^\ell$ \Comment{DDIM step towards $\pi^\theta$}
        \State Draw $M^\ell \sim \ber(q)$
        \If{$M^\ell = 1$}
            \State $v^\ell \gets \argmax_{v'} p^\ell_{v'}$
            \State $x^\ell \gets e_{v^\ell}$\Comment{top-$k$=1 sampling}
            \State $c^\ell \gets p^\ell_{v^\ell}$\Comment{Record confidence for remasking}
            \State $r^\ell \gets 1$
        \EndIf
    \EndFor
\EndFor
\State \Return $\left(\argmax_{v} x^{\ell}_v\right)_{\ell=1}^L$
\end{algorithmic}
\end{algorithm}

\section{Additional background}
\subsection{Discrete diffusion}
\label{app:discrete-diff-background}

Discrete diffusion models corrupt a sequence by acting directly on $V^L$, and learn to invert that corruption. We keep the flow matching time convention of the main text: $t = 1$ carries the data and $t = 0$ the noise, so the signal level $\gamma: [0,1]\to[0,1]$ is increasing with $\gamma_0 = 0$ and $\gamma_1 = 1$, exactly as for the flow schedules of~\Cref{sec:background}. Many works in the discrete diffusion literature use the opposite diffusion convention.

\paragraph{The forward process.} The two families below are instances of D3PM~\citep{d3pm}, which builds on~\citet{nonequilibrium,multinomialdiffusion,ddpm}, and whose continuous-time limits are continuous-time Markov chains~\citep{ctmc}. The corruption factorises over positions, and each position either keeps its token or is replaced by a draw from a fixed reference distribution $q_\mathrm{ref}$ on $V$:
\begin{equation}
\label{eq:app-dd-forward}
    q(X_t^\ell \mid \vx^\ell) = \gamma_t\,\delta_{\vx^\ell} + (1-\gamma_t)\,q_\mathrm{ref}.
\end{equation}
Masked diffusion models (MDMs) take $q_\mathrm{ref} = \delta_\masktok$, for an absorbing mask token $\masktok \notin V$, so that $X_t \in (V\cup\{\masktok\})^L$ while uniform-state diffusion models (USDMs) take $q_\mathrm{ref} = \unif(V)$, so that $X_t\in V^L$. The two choices give different objectives and different samplers.

\paragraph{Masked diffusion.} The denoiser takes the corrupted sequence and returns a distribution over $V$ at every position, \ie $\pi^{\theta,\ell}_t(X_t) \in \R^\vocabsize$. Two constraints make it exact: it puts no mass on $\masktok$, and it copies positions that are already unmasked, since the forward process never changes them~\citep{mdlm,md4}. The reverse process only unmasks at rate
\begin{equation}
\label{eq:app-mdm-rate}
    \hat{R}_t^\ell(\masktok, j) = \frac{\dot\gamma_t}{1-\gamma_t}\,\indic[j = \vx^\ell],
    \qquad
    \hat{R}_t^{\theta,\ell}(\masktok, j) = \frac{\dot\gamma_t}{1-\gamma_t}\,
    \pi^{\theta,\ell}_t(X_t)_j,
\end{equation}
and since $\pi^{\theta,\ell}_t(X_t)$ sums to one, the total outflow $\sum_j \hat{R}^{\theta,\ell}_t(\masktok,j)$ does not depend on $\theta$. Under these constraints the continuous-time bound collapses to a masked cross-entropy,
\begin{equation}
\label{eq:app-mdm-nelbo}
    -\log p_\theta(\vx) \leq -\int_0^1 \frac{\dot\gamma_t}{1-\gamma_t}\,
    \E_{X_t \sim q(\cdot\,\mid\,\vx)}\left[
    \sum_{\ell\,:\,X_t^\ell = \masktok}
    \log\left\langle X_1^\ell, \pi^{\theta,\ell}_t(X_t)\right\rangle
    \right]\mathrm{d}t,
\end{equation}
where $X_1^\ell = e_{\vx^\ell}$ as in~\Cref{sec:background}, and which, for the linear schedule $\gamma_t = t$, weighs the loss by $1/(1-t)$. Sampling reverses~\eqref{eq:app-dd-forward} by unmasking a subset of positions per step; a token, once revealed, is never revisited. MDMs are currently the most common setting of recent diffusion language models at scale~\citep{llada}.

\paragraph{Uniform-state diffusion.} With $q_\mathrm{ref} = \unif(V)$ the corrupted sequence contains no marker of which positions are wrong, so the model must both locate and correct the errors. Nothing is frozen, so the sampler may revise any token at any time rather than committing once. For completeness, we give the closed forms which are those of the uniform kernel of~\citet{d3pm,multinomialdiffusion}.

Let us write the marginal at position $\ell$, and the same quantity with the data replaced by the denoiser's prediction as
\begin{equation}
\label{eq:app-usdm-marginals}
    q_t^\ell(j) \coloneqq \gamma_t\,\indic[j = \vx^\ell] + \frac{1-\gamma_t}{\vocabsize},
    \qquad
    q_t^{\theta,\ell}(j) \coloneqq \gamma_t\,\pi^{\theta,\ell}_t(X_t)_j
    + \frac{1-\gamma_t}{\vocabsize}.
\end{equation}

Because the uniform kernels compose $K_a K_b = K_{ab}$, the transition from a cleaner time $s$ to a noisier time $t < s$ retains with probability $\gamma_t/\gamma_s$, and Bayes' rule gives the exact reverse posterior as an element wise product
\begin{equation}
\label{eq:app-usdm-posterior}
    q(X_s^\ell \mid X_t^\ell, \vx^\ell) \propto
    \left(\frac{\gamma_t}{\gamma_s}\,e_{X_t^\ell}
    + \Big(1 - \frac{\gamma_t}{\gamma_s}\Big)\pcatunif\right)
    \odot
    \Big(\gamma_s\,e_{\vx^\ell} + (1-\gamma_s)\pcatunif\Big),
\end{equation}
with $p_\theta(X_s^\ell\mid X_t^\ell)$ obtained by substituting $\pi^{\theta,\ell}_t(X_t)$ for $e_{\vx^\ell}$. This is the kernel a finite-step sampler uses. Its continuous-time limit is a jump process: writing $\nu_t \coloneqq \dot\gamma_t/\gamma_t$ for the corruption rate, the noising process leaves each position at rate $\nu_t$ towards $\unif(V)$, and time reversal gives, for $j \neq i$
\begin{equation}
\label{eq:app-usdm-rates}
    \hat{R}_t^\ell(i, j) = \frac{\nu_t}{\vocabsize}\cdot\frac{q_t^\ell(j)}{q_t^\ell(i)},
    \qquad
    \hat{R}_t^{\theta,\ell}(i, j)
    = \frac{\nu_t}{\vocabsize}\cdot\frac{q_t^{\theta,\ell}(j)}{q_t^{\theta,\ell}(i)}.
\end{equation}

The prior term of the ELBO vanishes, since $\gamma_0 = 0$ makes $q(X_0\mid\vx)$ exactly the uniform prior, and the remainder is the integrated KL rate between the two jump processes~\citep{ctmc}, which after dropping $\theta$-independent terms would become
\begin{equation}
\label{eq:app-usdm-nelbo}
    -\log p_\theta(\vx) \leq \frac{1}{\vocabsize}\int_0^1 \nu_t\,
    \E_{X_t\sim q(\cdot\,\mid\,\vx)}\left[
    \sum_{\ell=1}^L \sum_{j \neq X_t^\ell}
    \left(
    \frac{q_t^{\theta,\ell}(j)}{q_t^{\theta,\ell}(X_t^\ell)}
    -
    \frac{q_t^{\ell}(j)}{q_t^{\ell}(X_t^\ell)}
    \log \frac{q_t^{\theta,\ell}(j)}{q_t^{\theta,\ell}(X_t^\ell)}
    \right)
    \right]\mathrm{d}t
    + \mathrm{const}.
\end{equation}

Unlike~\eqref{eq:app-mdm-rate}, the total outflow
$\sum_{j\neq X_t^\ell}\hat{R}^{\theta,\ell}_t(X_t^\ell, j)$ \emph{does} depend on $\theta$, so the first term does not drop and the objective is not a cross-entropy: the model is scored on the ratios $q_t^{\theta,\ell}(j)/q_t^{\theta,\ell}(X_t^\ell)$ rather than on the probability it assigns to the true token, and every position contributes at every $t$, not
only the corrupted ones. See~\citet{udlm,audmloo} for practical estimators of~\eqref{eq:app-usdm-nelbo}, and \citet{sedd}, which sidesteps it by learning the ratios $\pdata(\vy)/\pdata(\vx)$ directly with a score-entropy loss.

\paragraph{Relation to the continuous case.} USDMs are the discrete family closest to the flows of~\Cref{sec:background}: \citet{duo} show that a Gaussian diffusion pushed through an $\argmax$ matches the marginals of a uniform state diffusion. Discrete flow matching~\citep{dfmgat,discretefmcampbell} reformulates both families as probability paths with a prescribed interpolant, in a formalism much closer to that of~\Cref{sec:background}.

\subsection{The VDM noise schedule}
\label{app:vdm-snr}

Plaid~\citep{plaid,replaid} is a variational diffusion model~\citep{vdm} on token embeddings, so the weight in~\eqref{eq:plaid-elbo} is the derivative of a learnt signal-to-noise ratio. We recall here where it comes from, and why its endpoints are treated differently from its shape.

Following~\Cref{sec:background}, Plaid corrupts the embedded sequence as $I_t = \alpha_t X_0 + \beta_t \emb(X_1)$, with $X_0 \sim \normal(\vzero,\matid)$ and $\emb$ applied position-wise, and sets
\begin{equation}
    \SNR(t) \coloneqq \frac{\beta_t^2}{\alpha_t^2}.
\end{equation}
VDM takes the path variance-preserving, $\alpha_t^2 + \beta_t^2 = 1$, so that $\beta_t^2 = \mathrm{sigmoid}(\log\SNR(t))$ and the whole path is determined by $\log\SNR$ alone which it parametrises with a monotone network in $t$. As shown in~\eqref{eq:plaid-elbo}, the continuous time bound then splits into a prior term, a reconstruction term, and a diffusion term with the diffusion term being
\begin{equation}
\label{eq:app-vdm-diff}
    \loss_\mathrm{diff}(\theta) = \frac{1}{2}\,\E_{t\sim\unif(0,1),\,X_0}\left[
    \SNR'(t)\,\norm{\embhead(I_t, t) - \emb(\vx)}^2\right],
\end{equation}
with $\SNR' > 0$. The schedule enters~\eqref{eq:app-vdm-diff} twice: through the weight $\SNR'(t)$ and through $I_t$. The second dependence carries no extra freedom, since $\beta_t^2 = \mathrm{sigmoid}(\log\SNR(t))$ and $\alpha_t^2 = 1 - \beta_t^2$ are functions of $\SNR(t)$ alone. Let us write
\begin{equation}
\label{eq:app-vdm-snr-interp}
    J_w \coloneqq \sqrt{\frac{1}{1+w}}\,X_0 + \sqrt{\frac{w}{1+w}}\,\emb(\vx)
\end{equation}
for the corrupted sequence at $\SNR$ level $w$, we have $I_t = J_{\SNR(t)}$: two schedules that reach the same $\SNR$ give the same noisy input, whatever time they reach it at. The denoiser's second argument behaves the same way as VDM conditions it on $\log\SNR(t)$ rather than on $t$; we write $\embhead(\cdot, w)$ for it below. Substituting $w = \SNR(t)$ so that $\mathrm{d}w = \SNR'(t)\,\mathrm{d}t$,
\begin{equation}
\label{eq:app-vdm-invariance}
    \loss_\mathrm{diff}(\theta) = \frac{1}{2}\int_{\SNR(0)}^{\SNR(1)}
    \E_{X_0}\norm{\embhead(J_w, w) - \emb(\vx)}^2\,
    \mathrm{d}w,
\end{equation}
so, at a fixed denoiser, the bound depends on the schedule through its two endpoints $\SNR(0)$ and $\SNR(1)$ and nothing else. VDM uses both halves of this statement: the endpoints are learnt by the ELBO, and the shape in between is learnt by minimising the variance of the Monte Carlo estimate of~\eqref{eq:app-vdm-diff}. Plaid inherits both, which is what we mean by learnt log-SNR endpoints in~\Cref{tab:tinygsm-embedded}.

\section{Additional experimental results}

\subsection{Hyperparameters and experimental setup}
\label{app:hyperparameters}
\paragraph{Hyperparameters.} We provide all our hyperparameters in~\Cref{tab:hyperparameters}. Throughout, we use the standard diffusion transformer~\citep[DiT][]{dit} adapted by~\citet{sedd} for discrete diffusion. For one-hot flows, we also use a simple linear projection in $\R^{\vocabsize\times h}$, where $h$ is the hidden dimension of the transformer, following~\citet{cfm,fmlm}. Models are evaluated with the EMA weights, except for in the large Nemotron case. We evaluate using the final checkpoint. OWT and TinyGSM checkpoints are trained using $8\times$ H100. Nemotron $1.7$B checkpoints are trained with $256\times$ H100. All models are trained with AdamW~\citep{adamw}; the hyperparameters for which are provided in~\Cref{tab:hyperparameters} too.

\paragraph{L2 Weighting.} For one-hot flows (including base one-hot, CANDI, \doublecat, \semicat), for training, we weight the loss as in~\citet{scalingcfm}, following~\citet{dfm,geng2024consistencymodelseasy}, inversely proportional to the square root of the $L_2$ distance of the prediction to the target. Precisely, we divide the loss by $\stopgrad(\Delta)$ for $\Delta = \sqrt{\norm{\pi_t^{\theta,\ell}(I_t) - X_1^\ell} + \epsilon}$, with $\epsilon = 10^{-2}$.

\paragraph{OWT setup.} We use the typical OWT experimental setup: $1024$ long sequences, tokenised by the tokeniser of GPT-2~\citep{gpt2}.

\paragraph{TinyGSM.} The TinyGSM dataset~\citep{tinygsm} contains $11.8$M grade school mathematics questions paired with Python code traces that generate the correct solution. Our default hyperparameters are provided in~\Cref{tab:hyperparameters}. Since those differ slightly from the standard settings in the literature, we also use the standard ones for our final \semicat run in~\Cref{tab:hyperparameters-tinygsm-standard}. See~\Cref{tab:hyperparameters-tinygsm-standard} for a comparison.
\begin{table}[tbp]
\centering
\small
\begin{tabular}{lccc}
\toprule
& OpenWebText & TinyGSM & Nemotron ($1.7$B) \\
\midrule
\multicolumn{4}{l}{\emph{Architecture (DiT)}} \\
Hidden size & $768$ & $768$ & $2304$ \\
adaLN cond.\ dim & $128$ & $128$ & $128$ \\
Layers & $12$ & $12$ & $18$ \\
Attention heads & $12$ & $12$ & $18$ \\
FFN (SwiGLU) multiplier & $4$ & $4$ & $3$ \\
Dropout & $0$ & $0$ & $0$ \\
Sequence length $L$ & $1024$ & $512$ & $2048$ \\
Vocabulary $\vocabsize$ & $50\,257$ (GPT-2) & $151\,680$ (Qwen2) & $100\,352$ (TikToken) \\
Parameters, trunk & $121$M & $121$M & $1\,281$M \\
Parameters, in $+$ out proj. & $77$M & $233$M & $462$M \\
Parameters, total & $199$M & $354$M & $1\,743$M \\
\midrule
\multicolumn{4}{@{}l}{\emph{Optimisation (AdamW)}} \\
Optimiser & \multicolumn{2}{c}{($\beta_1{=}0.9$, $\beta_2{=}0.999$, $\epsilon{=}10^{-8}$)} & ($\beta_1{=}0.9$, $\beta_2{=}0.95$, $\epsilon{=}10^{-8}$) \\
Weight decay & \multicolumn{2}{c}{$0.1$} & $0.1$ \\
Gradient clipping & \multicolumn{2}{c}{$1.0$} & $1.0$ \\
Schedule & \multicolumn{2}{c}{Cosine decay w/ linear warmup} & Cosine decay w/ linear warmup \\
Warmup steps & \multicolumn{2}{c}{$3\,000$} & $2\,000$ \\
Peak / final LR & \multicolumn{2}{c}{$3{\times}10^{-4}$ / $10^{-5}$} & $3{\times}10^{-4}$ / $3{\times}10^{-5}$ \\
Total steps & $1\,000\,000$ & $200\,000$ & $400\,457$ \\
Global batch size & $512$ & $256$ & $2\,560$ \\
Precision & \multicolumn{2}{c}{bfloat16 mixed} & bfloat16 mixed \\
EMA decay & \multicolumn{2}{c}{$0.9999$} & none \\
\bottomrule
\end{tabular}
\caption{Hyperparameters of our models on OpenWebText, TinyGSM and Nemotron.}
\label{tab:hyperparameters}
\end{table}

\begin{table}[tbp]
\centering
\small
\begin{tabular}{lcc}
\toprule
& TinyGSM (ours) & TinyGSM (standard) \\
\midrule
Global batch size & $256$ & $512$ \\
Total steps & $200\,000$ & $250\,000$ \\
Schedule & Cosine decay w/ linear warmup & Constant w/ linear warmup \\
Warmup steps & $3\,000$ & $2\,500$ \\
Peak / final LR & $3{\times}10^{-4}$ / $10^{-5}$ & $3{\times}10^{-4}$ / $3{\times}10^{-4}$ \\
Weight decay & $0.1$ & $0$ \\
Optimiser & \multicolumn{2}{c}{AdamW ($\beta_1{=}0.9$, $\beta_2{=}0.999$, $\epsilon{=}10^{-8}$)} \\
Gradient clipping & \multicolumn{2}{c}{$1.0$} \\
Precision & \multicolumn{2}{c}{bfloat16 mixed} \\
EMA decay & \multicolumn{2}{c}{$0.9999$} \\
Architecture & \multicolumn{2}{c}{As in~\Cref{tab:hyperparameters}} \\
\bottomrule
\end{tabular}
\caption{A comparison of TinyGSM hyperparameters between ours and the standard ones found in the literature. We run most of our experiments on our own, and report \semicat's numbers on the standard configuration too (\Cref{tab:tinygsm-doublecat}).}
\label{tab:hyperparameters-tinygsm-standard}
\end{table}

\subsection{DoubleCat ablations}

The ablations over the choice of prior and the choice of schedule for \doublecat are available in~\Cref{tab:doublecat-prior,tab:doublecat-schedule}. We also provide the full ablation of \doublecat and one-hot flow with self-conditioning on and off in~\Cref{tab:self-conditioning-doublecat}.
\begin{table}[tbp]
\centering
\small
\begin{tabular*}{\textwidth}{@{\extracolsep{\fill}}lcccccc@{}}
\toprule
& \multicolumn{4}{c}{Accuracy (\%)} & \multicolumn{2}{c}{Best (\%)} \\
\cmidrule(lr){2-5}\cmidrule(lr){6-7}
$Y_0^\ell$ & $T = 1$ & $T = 0.5$ & $T = 0.2$ & $\mathrm{top}\text{-}k\text{=}1$ & Runs ($\uparrow$) & Name err. ($\downarrow$) \\
\midrule
Barycentre & $4.70$ & $6.52$ & $6.29$ & $7.51$ & $\mathbf{55.57}$ & $\mathbf{43.90}$ \\
Uniform & $\mathbf{6.14}$ & $\mathbf{9.17}$ & $\mathbf{10.84}$ & $\mathbf{11.30}$ & $53.53$ & $46.10$ \\
Uniform vertex & $5.31$ & $7.05$ & $7.81$ & $8.49$ & $51.48$ & $47.99$ \\
\bottomrule
\end{tabular*}
\caption{Ablation on the simplex prior $Y_0^\ell$ for \doublecat, at NFE $256$. Accuracy is reported per sampler; the last two columns report the best value attained over those samplers. ``Barycentre'' is $\delta_{\pcatunif}$, ``Uniform'' is uniform over the simplex, $\unif(\Delta^\vocabsize)$, and ``uniform vertex'' is uniform over the vertices of the simplex (\ie the one-hot encoding of $\unif(V)$). The inner schedule is set to $\kappa_t \equiv t^2$. We retain the uniform distribution which consistently outperforms the other choices on accuracy.}
\label{tab:doublecat-prior}
\end{table}

\begin{table}[tbp]
\centering
\small
\begin{tabular*}{\textwidth}{@{\extracolsep{\fill}}lcccccc@{}}
\toprule
& \multicolumn{4}{c}{Accuracy (\%)} & \multicolumn{2}{c}{Best (\%)} \\
\cmidrule(lr){2-5}\cmidrule(lr){6-7}
$\kappa_t$ & $T = 1$ & $T = 0.5$ & $T = 0.2$ & $\mathrm{top}\text{-}k\text{=}1$ & Runs ($\uparrow$) & Name err. ($\downarrow$) \\
\midrule
$\sqrt{t}$ & $\mathbf{9.70}$ & $\mathbf{13.50}$ & $\mathbf{14.18}$ & $\mathbf{14.78}$ & $\mathbf{63.61}$ & $\mathbf{36.16}$ \\
$t$ & $6.97$ & $9.55$ & $10.61$ & $10.84$ & $55.19$ & $44.43$ \\
$t^2$ & $6.14$ & $9.17$ & $10.84$ & $11.30$ & $53.53$ & $46.10$ \\
$t^3$ & $5.00$ & $7.35$ & $7.05$ & $8.79$ & $51.18$ & $48.45$ \\
\bottomrule
\end{tabular*}
\caption{Ablation on the simplex schedule $\kappa_t$ for \doublecat, at NFE $256$, with $Y_0^\ell \sim \unif(\Delta^\vocabsize)$ for all runs. Accuracy is reported per sampler; the last two columns report the best value attained over those samplers.}
\label{tab:doublecat-schedule}
\end{table}

\begin{table}[tbp]
  \centering
  \small
  \setlength{\tabcolsep}{4pt}
  \begin{tabular*}{\textwidth}{@{\extracolsep{\fill}}llcccccc@{}}
    \toprule
    & & \multicolumn{4}{c}{\textbf{Accuracy} (\%, $\uparrow$)} & \multicolumn{2}{c}{\textbf{Best} (\%)} \\
    \cmidrule(lr){3-6}\cmidrule(lr){7-8}
    Method & Self-cond. & $T = 1$ & $T = 0.5$ & $T = 0.2$ & $\mathrm{top}\text{-}k\text{=}1$ & Runs ($\uparrow$) & Name err. ($\downarrow$) \\
    \midrule
    \multirow{2}{*}{One-hot flow} & Off & $3.28$ & $4.42$ & $3.89$ & $3.64$ & $50.91$ & $48.07$ \\
     & On & $\mathbf{6.44}$ & $\mathbf{6.75}$ & $\mathbf{8.34}$ & $\mathbf{6.97}$ & $\mathbf{61.79}$ & $\mathbf{37.60}$ \\
    \midrule
    \multirow{2}{*}{\doublecat} & Off & $11.45$ & $15.85$ & $16.00$ & $15.92$ & $69.67$ & $30.02$ \\
     & On & $\mathbf{16.83}$ & $\mathbf{20.77}$ & $\mathbf{20.70}$ & $\mathbf{19.79}$ & $\mathbf{75.21}$ & $\mathbf{24.34}$ \\
    \bottomrule
  \end{tabular*}
  \caption{Self-conditioning effect on one-hot flows and \doublecat for schedule $\kappa_t \equiv \sqrt{t}$.}
  \label{tab:self-conditioning-doublecat}
\end{table}

\subsection{OpenWebText Pareto Frontiers}
We provide additional OWT Pareto frontiers in~\Cref{fig:plaid-pareto-fronts-dims,fig:cross-pareto-front-owt}.
\begin{figure}[tbp]
    \centering
    \includegraphics[width=\linewidth]{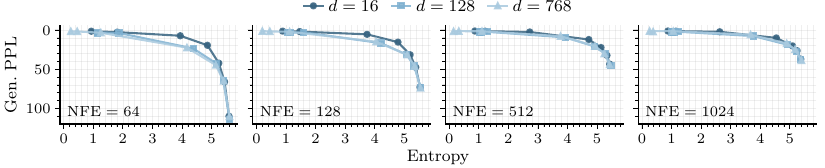}
    \caption{\genppl as measured by GPT-2 for models trained on OpenWebText with sequence lengths of $1024$ tokens, using the GPT-2 tokeniser. All Plaid runs for different embedding dimensionalities, $d$.}
    \label{fig:plaid-pareto-fronts-dims}
\end{figure}
\begin{figure}[tbp]
    \centering
    \includegraphics[width=\linewidth]{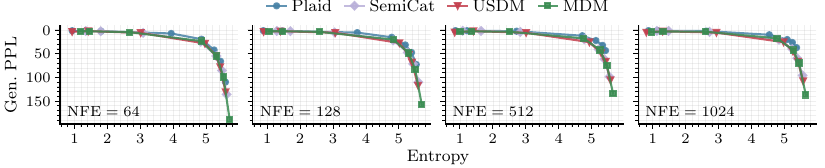}
    \caption{\genppl as measured by GPT-2 for models trained on OpenWebText with sequence lengths of $1024$ tokens, using the GPT-2 tokeniser. Plaid is for $d=16$.}
    \label{fig:cross-pareto-front-owt}
\end{figure}

\subsection{Corruption as data augmentation}
We show, here, that corruption seen as data augmentation, instead of a proper change in the interpolant, lowers the performance of the resulting model on the unconditional generation task of OWT by substantially diminishing its diversity (entropy). To be precise, for $U^\ell \sim \unif(V)$ i.i.d., we use interpolants instead pointing towards
\begin{equation}
    \tilde{X}_1^\ell = \zeta_t X_1^\ell + (1 - \zeta_t)U^\ell,
\end{equation}
with various choices of $\zeta_t \sim \ber(1-\rho_t)$, constant and scheduled. Results are provided in~\Cref{tab:corruption-augmentation}. Any score comparable in \genppl and entropy is found for very low constant choices of $\rho_t$; lower (better) \genppl scores imply entropy collapse.

\begin{table}[tbp]
\centering
\small
\begin{tabular*}{\textwidth}{@{\extracolsep{\fill}}lccc@{}}
\toprule
& \multicolumn{3}{c}{\genppl ($\downarrow$) (entropy)} \\
\cmidrule(lr){2-4}
Flip probability $\rho_t$ & NFE $= 50$ & NFE $= 100$ & NFE $= 200$ \\
\midrule
\emph{Baseline, $1$M steps, NFE $= 128$}\\
\quad One-hot flow & \multicolumn{3}{c}{$85.64~(5.45)$} \\
\quad \doublecat & \multicolumn{3}{c}{$66.02~(5.38)$} \\
\quad \semicat & \multicolumn{3}{c}{$117.00~(5.62)$} \\
\midrule
\emph{no $L_2$ weighting}\\
\quad $\rho_t \equiv 0.01$ & $127.70~(5.44)$ & $118.74~(5.52)$ & $115.59~(5.51)$ \\
\quad $\rho_t \equiv 0.1$ & $91.58~(5.09)$ & $81.71~(5.19)$ & $79.62~(5.12)$ \\
\quad $\rho_t \equiv 0.25$ & $68.97~(4.72)$ & $61.64~(4.83)$ & $65.67~(4.88)$ \\
\quad $\rho_t \equiv 0.5$ & $48.94~(4.13)$ & $44.23~(4.16)$ & $47.68~(4.18)$ \\
\quad $\rho_t = t$ & $22.88~(3.12)$ & $22.07~(3.30)$ & $23.13~(3.18)$ \\
\quad $\rho_t = \gamma_t$ & $26.80~(2.86)$ & $22.41~(2.86)$ & $24.29~(2.99)$ \\
\midrule
\emph{with $L_2$ weighting}\\
\quad $\rho_t \equiv 0.01$ & $100.44~(5.42)$ & $86.72~(5.43)$ & $85.33~(5.45)$ \\
\quad $\rho_t \equiv 0.1$ & $75.00~(5.14)$ & $65.62~(5.19)$ & $63.50~(5.18)$ \\
\quad $\rho_t \equiv 0.25$ & $57.64~(4.84)$ & $53.68~(4.89)$ & $51.50~(4.89)$ \\
\quad $\rho_t \equiv 0.5$ & $44.43~(4.28)$ & $39.28~(4.30)$ & $41.99~(4.31)$ \\
\quad $\rho_t = t$ & $27.99~(3.37)$ & $28.55~(3.46)$ & $29.26~(3.36)$ \\
\quad $\rho_t = \gamma_t$ & $24.60~(3.00)$ & $21.63~(2.98)$ & $23.82~(3.13)$ \\
\bottomrule
\end{tabular*}
\caption{OWT \genppl (GPT-2), token entropy in parentheses, at $200$k steps, temperature $1$. The corruption flips each $X_1^\ell$ independently, with probability $\rho_t$, to a token drawn uniformly from $V$; it enters the interpolant only, leaving the loss target clean. Trained with $\gamma$ being the usual mixture schedule.}
\label{tab:corruption-augmentation}
\end{table}

\subsection{General hybrid corruption processes}
\label{app:hybrid-methods}
We also explore the design of hybrid methods based on continuous flows. Inspired by AUDM~\citep{audmloo}, we introduce \alphacat, a generalisation of CANDI~\citep{candi}, where the model is instead exposed to the following interpolants: letting $\tau^\ell\sim \unif(0,1)$ i.i.d., as before, and $U^\ell \mid X_1^\ell = \vx^\ell \sim \cat(\alpha\pcatunif + (1-\alpha) e_{\vx^\ell})$ i.i.d., we define the interpolants to be, for each $\ell$, for $\tau^\ell > t$,
\begin{equation}
    I_t^\ell(X_0, X_1) = (1-\gamma_t)X_0^\ell + \gamma_t e_{U^\ell},
\end{equation}
and, for $\tau^\ell \leq t$,
\begin{equation}
    I_t^\ell(X_0, X_1) = X_1^\ell.
\end{equation}
The model is thus effectively exposed to various levels $\alpha$ of additional discrete corruption, recovering CANDI at $\alpha = 0$. Simply put, at each position, the model flows towards either a random vertex with probability $\alpha$ or to the true one otherwise, until its reveal time $\tau^\ell$. We report the TinyGSM results in~\Cref{tab:alphacat-tinygsm}. As we can see, with increasing $\alpha$, we recover better performance until $\alpha=0.75$ as it recovers the effect of \semicat. Nonetheless, for $\alpha = 1$, the model learns a leave-one-out denoiser; that is to say, it learns $\E[X_1^\ell \mid I_t^{-\ell} = x_t^{-\ell}]$, which ignores position $\ell$ to predict token $\ell$, like discrete diffusion methods~\citep{audmloo}. This is natural, since $\MI(I_t^\ell; X_1^\ell) = 0$ for $t < \tau^\ell$ and $\alpha = 1$, where $\MI$ denotes mutual information. Variants with $\alpha > 0$ thus partially defy the purpose of hybrid methods, as they rely on discretely revealed positions.

\begin{table}[t]
  \centering
  \small
  \setlength{\tabcolsep}{4pt}
  \begin{tabular*}{\textwidth}{@{\extracolsep{\fill}}lcccccc@{}}
    \toprule
    & \multicolumn{4}{c}{\textbf{Accuracy} (\%, $\uparrow$)} & \multicolumn{2}{c}{\textbf{Best} (\%)} \\
    \cmidrule(lr){2-5}\cmidrule(lr){6-7}
    $\alpha$ & $T = 1$ & $T = 0.5$ & $T = 0.2$ & $\mathrm{top}\text{-}k\text{=}1$ & Runs ($\uparrow$) & Name err. ($\downarrow$) \\
    \midrule
    $0$ & $7.20$ & $9.55$ & $9.63$ & $8.72$ & $69.22$ & $29.26$ \\
    $0.25$ & $13.50$ & $18.35$ & $17.89$ & $17.97$ & $79.38$ & $19.86$ \\
    $0.5$ & $\mathbf{19.26}$ & $23.88$ & $25.17$ & $25.93$ & $84.91$ & $14.71$ \\
    $0.75$ & $15.47$ & $\mathbf{25.25}$ & $\mathbf{29.57}$ & $\mathbf{29.26}$ & $\mathbf{87.11}$ & $\mathbf{12.81}$ \\
    $1$ & $14.56$ & $23.86$ & $27.74$ & $27.37$ & $86.73$ & $\mathbf{12.81}$ \\
    \bottomrule
  \end{tabular*}
  \caption{TinyGSM accuracy at NFE $=256$, for \alphacat{} at corruption levels $\alpha \in \{0, 0.25, 0.5, 0.75, 1\}$.}
  \label{tab:alphacat-tinygsm}
\end{table}

\subsection{Acceleration via D-MMD}
\label{app:dmmd}
We apply D-MMD distillation~\citep{dmmd,mmd} to our pre-trained \semicat model. We tune the temperature of the teacher $T$, and attempt some simple techniques such as sampling boundary pairs $(s,t) = (0,1)$ with a probability $p$, and a gap curriculum updated every $n$ steps, both inspired by~\citet{fmlm}. We find that the resulting model is competitive with FMLM+~\citep{fmlmplus}, and tends to have higher accuracy in the $1$-step and $32$-step case, while FMLM+ retains a consistent and balanced accuracy for different numbers of steps. Our results are obtained with simple naive sampling, and we hypothesise that more advanced samplers could further improve the performance (\eg confidence-based).

\begin{table}[t]
  \centering
  \small
  \setlength{\tabcolsep}{4pt}
  \begin{tabularx}{\textwidth}{@{}lYYYYYY@{}}
    \toprule
    & \multicolumn{6}{c}{\textbf{Accuracy} (\%, $\uparrow$)} \\
    \cmidrule(lr){2-7}
    Student & NFE $1$ & $2$ & $4$ & $8$ & $16$ & $32$ \\
    \midrule
    FMLM+~\citep{fmlmplus} & $0.2$ & $1.2$ & $5.7$ & $12.4$ & $17.8$ & $19.9$ \\
    \midrule
    \semicat teacher & -- & -- & -- & $5.5$ & $12.1$ & $20.1$ \\
    \midrule
    $T = 1$ & $0.2$ & $0.5$ & $1.6$ & $5.4$ & $13.0$ & $19.0$ \\
    $T = 0.5$ & $0.7$ & $1.2$ & $2.8$ & $9.8$ & $16.5$ & $24.0$ \\
    $T = 0.3$ & $0.5$ & $0.8$ & $3.9$ & $9.2$ & $16.3$ & $22.7$ \\
    $T = 0.2$ & $0.6$ & $1.8$ & $4.8$ & $10.9$ & $16.3$ & $22.5$ \\
    \hlrow $T = 0.05$ & $\mathbf{1.1}$ & $1.8$ & $3.6$ & $9.0$ & $17.1$ & $\mathbf{24.2}$ \\
    \hlrow \quad $+$ boundary pairs $3\%$ & $\mathbf{1.1}$ & $\mathbf{2.5}$ & $\mathbf{5.6}$ & $9.6$ & $15.9$ & $20.0$ \\
    \hlrow \quad $+$ boundary pairs $25\%$ & $0.6$ & $1.1$ & $1.9$ & $4.9$ & $10.0$ & $15.5$ \\
    \hlrow \quad $+$ gap curriculum $2$k & $\mathbf{1.1}$ & $2.4$ & $5.4$ & $\mathbf{11.7}$ & $\mathbf{19.1}$ & $22.1$ \\
    \hlrow \quad $+$ gap curriculum $2$k, $3\%$ & $0.6$ & $2.0$ & $3.9$ & $6.7$ & $13.8$ & $20.0$ \\
    \hlrow \quad $+$ gap curriculum $5$k, $3\%$ & $0.7$ & $1.5$ & $3.7$ & $7.8$ & $12.4$ & $20.4$ \\
    \bottomrule
  \end{tabularx}
  \caption{TinyGSM accuracy for D-MMD distillation of \semicat, at teacher temperature $T$. Students train for $20$k steps at constant learning rate $2\times10^{-5}$, with early stopping on GSM8K test set accuracy at $32$ NFE.}
  \label{tab:dmmd-tinygsm}
\end{table}

\section{Embedded spaces}
\subsection{Reproducing SemiCat's statistics on an embedded space}
\label{app:doublecat-embed-reprod}
Let us formally write out the interpolation incurred by \semicat. Given i.i.d.\ times $\tau^\ell \sim \unif(0, 1)$, at each position, we have, for $0 \leq t < \tau^\ell$,
\begin{equation}
I_t^\ell(X_0, X_1) = (1-\gamma_t) X_0^\ell + \gamma_t Y_t^\ell,
\end{equation}
and, for $t \geq \tau^\ell$,
\begin{equation}
I_t^\ell(X_0, X_1) = X_1^\ell.
\end{equation}
Thus, conditioning on the same times, we have either, for unrevealed times,
\begin{equation}
    W^\top I_t^\ell(X_0, X_1) \sim \normal\left( \gamma_t W^\top Y_t^\ell, (1-\gamma_t)^2 W^\top W \right).
\end{equation}
Moreover, since
\begin{equation}
    Y_t^\ell = (1-\kappa_t) Y_0^\ell + \kappa_t X_1^\ell,
\end{equation}
for a Dirichlet prior, the mean conditional on $X_1^\ell$ is
\begin{equation}
    (1-\kappa_t)W^\top\pcatunif + \kappa_t W^\top X_1^\ell,
\end{equation}
the first term corresponding to the barycentre of the embeddings. For times after reveal, $W^\top I_t^\ell(X_0, X_1) = W^\top X_1^\ell = \emb(X_1^\ell)$. The only difficulty to obtain the same statistics directly in the embedded space is to obtain the same covariance structure. To do so efficiently, we can compute $W^\top W$, which we can then decompose by Cholesky as $W^\top W = LL^\top$, and thus apply it to a Gaussian draw in lower dimension, $LZ$, for $Z\sim\normal(0, \matid)$, thus obtaining a covariance of $LL^\top = W^\top W$. Our trained models obtain the same final performance as normal \semicat training, but it is less stable due to the positive definiteness requirement of Cholesky.

\subsection{Embedding-dependent schedules}
\label{app:emb-dep-sched}
The decoding error rate for one-hot encoded data has the well-known form from~\citet{candi,duo}. We give here the building block needed to generalise it to arbitrary embeddings: the pairwise nearest-neighbour comparison. The full DER then follows from the joint behaviour of these $\vocabsize - 1$ comparisons, which we do not treat here as they are not independent.

\begin{propbox}
\begin{restatable}[Embedding-dependent pairwise decoding scores]{proposition}{emb-dep-sched}\label{prop:emb-dep-sched}
Given an embedding matrix $W \in \R^{\vocabsize \times d}$, and the interpolant $I_t^\ell = (1-\gamma_t)X_0^\ell + \gamma_t \emb(X_1^\ell)$ between a standard Gaussian and the token's embedding, define the nearest-neighbour decoding scores $S_t(v) \coloneqq \norm{\emb(v) - I_t^\ell}^2$. Then, for any pair of tokens $i \neq j$, writing $\Delta_{ij} \coloneqq \emb(j) - \emb(i)$,
\begin{equation}
    \big(S_t(i) - S_t(j)\big) \sim \normal\left(m_{ij}(t),\, 4(1-\gamma_t)^2\norm{\Delta_{ij}}^2\right),
\end{equation}
and hence, since decoding is an $\argmin$ over squared distances, token $i$ is preferred over token $j$ with probability
\begin{equation}
    \forall 0 \leq t \leq 1, \qquad \Prob\left(S_t(i) < S_t(j)\right) = 1 - \Phi\left(\frac{m_{ij}(t)}{2(1-\gamma_t)\norm{\Delta_{ij}}}\right),
\end{equation}
where $\Phi$ is the c.d.f.\ of the standard normal distribution and the mean score gap is
\begin{equation}
    m_{ij}(t) \coloneqq 2\gamma_t\left\langle \Delta_{ij}, \emb(X_1^\ell) \right\rangle + \norm{\emb(i)}^2 - \norm{\emb(j)}^2.
\end{equation}
\end{restatable}
\end{propbox}
\begin{proof}
Let $\ell \in \intrange{1}{L}$. We have that
\begin{equation}
    I_t^\ell \sim \normal\left(\gamma_t \emb(X_1^\ell), (1-\gamma_t)^2\matid\right).
\end{equation}
For all $v \in V$, define the scores
\begin{equation}
    S_t(v) = \norm{\emb(v) - I_t^\ell}^2 = \norm{\emb(v)}^2 + \norm{I_t^\ell}^2 - 2\left\langle \emb(v), I_t^\ell \right\rangle.
\end{equation}
For the nearest neighbour decoding to yield the true final token, $X_1^\ell$, let us compare the pairwise errors, for indices $i \neq j$:
\begin{equation}
    S_t(i) - S_t(j) = \norm{\emb(i)}^2 - \norm{\emb(j)}^2 - 2\left\langle\emb(i), I_t^\ell\right\rangle + 2\left\langle\emb(j), I_t^\ell\right\rangle.
\end{equation}
Expanding the dot product with embedding $i$ (or identically for $j$), by linearity,
\begin{equation}
    \left\langle\emb(j), I_t^\ell\right\rangle = (1-\gamma_t)\left\langle\emb(j), X_0^\ell\right\rangle + \gamma_t\left\langle \emb(j), \emb(X_1^\ell) \right\rangle,
\end{equation}
we can then rewrite the difference of scores, for $\Delta_{ij} \coloneqq \emb(j) - \emb(i)$, as
\begin{equation}
    S_t(i) - S_t(j) = 2\gamma_t\left\langle \Delta_{ij}, \emb(X_1^\ell) \right\rangle + \norm{\emb(i)}^2 - \norm{\emb(j)}^2 + 2(1-\gamma_t)\left\langle \Delta_{ij}, X_0^\ell \right\rangle.
\end{equation}
Since $X_0^\ell \sim \normal(0, \matid)$, the last term follows $\normal\left(0, 4(1-\gamma_t)^2 \norm{\Delta_{ij}}^2\right)$. Hence,
\begin{equation}
    \big(S_t(i) - S_t(j)\big)\sim\normal\left(m_{ij}(t), 4(1-\gamma_t)^2 \norm{\Delta_{ij}}^2 \right),
\end{equation}
where
\begin{equation}
    \forall 0\leq t \leq 1,\qquad m_{ij}(t) \coloneqq 2\gamma_t\left\langle \Delta_{ij}, \emb(X_1^\ell) \right\rangle + \norm{\emb(i)}^2 - \norm{\emb(j)}^2.
\end{equation}
So the probability of the event of interest can be expressed as
\begin{equation}
    \Prob\left(S_t(i) < S_t(j)\right) = \Phi\left(\frac{-m_{ij}(t)}{2(1-\gamma_t)\norm{\Delta_{ij}}}\right) = 1 - \Phi\left(\frac{m_{ij}(t)}{2(1-\gamma_t)\norm{\Delta_{ij}}}\right).
\end{equation}
\end{proof}

\paragraph{Recovering the one-hot DER.}
\label{app:one-hot-limit}
Informally, consider unit-norm embeddings with pairwise inner products $\rho_{ij} \coloneqq \langle \emb(i), \emb(j)\rangle$, $i \neq j$. If cross-entropy spreads the embeddings roughly uniformly over the sphere, then, since two independent uniform points on the sphere have an inner product of mean $0$ and variance $1/d$, the $\rho_{ij}$ are of order $1/\sqrt{d}$, and thus vanish as $d$ grows. Taking $i = X_1^\ell$ to be the true token, the norm terms of~\Cref{prop:emb-dep-sched} cancel, $\norm{\Delta_{ij}}^2 = 2(1-\rho_{ij})$ and $m_{ij}(t) = -2\gamma_t(1-\rho_{ij})$, so that
\begin{equation}
    \Prob\left(S_t(i) < S_t(j)\right) = \Phi\left(\frac{\gamma_t}{1-\gamma_t}\sqrt{\frac{1-\rho_{ij}}{2}}\right) \xrightarrow[\rho_{ij} \to 0]{} \Phi\left(\frac{\gamma_t}{\sqrt{2}(1-\gamma_t)}\right),
\end{equation}
the one-hot value. In the same limit, the comparisons also recover the correlation structure of the one-hot case, up to $O(\rho_{ij})$ corrections, so the full DER tends to the one-hot expression of~\citet{candi,duo}. An embedding spread out by cross-entropy therefore approximately inherits the one-hot DER.

\subsection{Cross-entropy, mean squared error and the embedding geometry}
\label{app:ce-mse-geometry}
To gain some intuition, we compare, informally and approximately, how the cross-entropy and the MSE terms of Plaid~\citep{plaid} move the input embeddings $W \in \R^{\vocabsize \times d}$, with rows $\emb(v)$. Fix a position $\ell$ with true token $i = \vx^\ell$, write $\sigma_t = 1-\gamma_t$ and $I_t^\ell = \sigma_t X_0^\ell + \gamma_t \emb(i)$. \emph{We assume the denoiser is Bayes-optimal}, so that its prediction is the posterior
\begin{equation}
    \pi^{\theta,\ell,v}_t(I_t) \propto p(v) \exp\left(-\norm{I_t^\ell - \gamma_t\emb(v)}^2/2\sigma_t^2\right),
\end{equation}
for some prior $p$ over tokens. Denote by $\mu \coloneqq \pi^{\theta,\ell}_t W = \sum_v \pi^{\theta,\ell,v}_t \emb(v)$ the expected embedding, and observe that $\mu - \emb(i) = \sum_{v \neq i} \pi^{\theta,\ell,v}_t (\emb(v) - \emb(i))$ points \emph{towards} the tokens carrying the error mass.

The cross-entropy sees the embeddings only through the interpolant, through $\emb(i)$, and recall that $X_1^\ell$ is the one-hot random variable for the token at position $\ell$ (\ie $X_1^{\ell, \vx^\ell} = 1$ and $X_1^{\ell, v}= 0$ for $v \neq \vx^\ell$). Hence, omitting the arguments for the denoiser at time $t$, $\pi^\theta_t$, for conciseness,
\begin{align}
    \nabla_{\emb(i)}\left(-\log \pi^{\theta,\ell,i}_t\right) &= \gamma_t \nabla_{I_t^\ell}\left(-\log \pi^{\theta,\ell,i}_t\right)\\
    &= \gamma_t \nabla_{I_t^\ell}\Biggl(\frac{\norm{I_t^\ell - \gamma_t\emb(i)}^2}{2\sigma_t^2} \notag\\
    &\qquad + \log \sum_{v} p(v) \exp\left(-\frac{\norm{I_t^\ell - \gamma_t\emb(v)}^2}{2\sigma_t^2}\right)\Biggr) \\
    &= \frac{\gamma_t}{\sigma_t^2}\left(\left(I_t^\ell - \gamma_t\emb(i)\right) - \sum_{v} \pi^{\theta,\ell,v}_t \cdot\left(I_t^\ell - \gamma_t\emb(v)\right)\right) \label{eq:ce-grad-posterior}\\
    &= \frac{\gamma_t}{\sigma_t^2}\left(\sum_{v} X_1^{\ell,v}\left(I_t^\ell - \gamma_t\emb(v)\right) - \sum_{v} \pi^{\theta,\ell,v}_t \cdot \left(I_t^\ell - \gamma_t\emb(v)\right)\right) \\
    &= \frac{\gamma_t}{\sigma_t^2}\sum_{v}\left(X_1^{\ell,v} - \pi^{\theta,\ell,v}_t\right)\left(I_t^\ell - \gamma_t\emb(v)\right) \\
    &= \frac{\gamma_t}{\sigma_t^2}\Biggl(I_t^\ell \underbrace{\sum_{v}\left(X_1^{\ell,v} - \pi^{\theta,\ell,v}_t\right)}_{=0} - \gamma_t \sum_{v}\left(X_1^{\ell,v} - \pi^{\theta,\ell,v}_t\right)\emb(v)\Biggr)\\
    \nabla_{\emb(i)}\left(-\log \pi^{\theta,\ell,i}_t\right) &= \frac{\gamma_t^2}{\sigma_t^2}\left(\mu - \emb(i)\right).
\end{align}
For~\eqref{eq:ce-grad-posterior}, we used the Bayes optimality of the denoiser. Writing $a_v \coloneqq \norm{I_t^\ell - \gamma_t\emb(v)}^2/2\sigma_t^2$, so that $\nabla_{I_t^\ell} a_v = (I_t^\ell - \gamma_t\emb(v))/\sigma_t^2$,
\begin{equation}
    \nabla_{I_t^\ell} \log \sum_{v} p(v) e^{-a_v} = -\sum_{v} \frac{p(v) e^{-a_v}}{\sum_{u} p(u) e^{-a_u}} \nabla_{I_t^\ell} a_v = -\frac{1}{\sigma_t^2}\sum_{v} \pi^{\theta,\ell,v}_t\cdot\left(I_t^\ell - \gamma_t\emb(v)\right),
\end{equation}
with the last equality coming from the definition of the Bayes optimal denoiser. The wrong embeddings, $\emb(v)$ for $v \neq i$, receive no gradient from the cross-entropy through the input, since they do not enter $I_t^\ell$. On the other hand, for the MSE term, $\tfrac{1}{2} w_t \norm{(\pi^{\theta,\ell}_t - X_1^\ell)W}^2$, with weight $w_t \geq 0$ for all $t$,
\begin{equation}
    \nabla_{\emb(v)}\, \tfrac{1}{2} w_t \norm{(\pi^{\theta,\ell}_t - X_1^\ell)W}^2 = w_t\left(\pi^{\theta,\ell,v}_t - X_1^{\ell,v}\right)\left(\mu - \emb(i)\right) + w_t\nabla_{\emb(v)}\pi^{\theta,\ell}_t W\left(\mu - \emb(i)\right).
\end{equation}
The cross-entropy gradient and the first term of the MSE gradient act on $\emb(i)$ along the same direction, $\mu - \emb(i)$, but with opposite signs. A descent step on the cross-entropy moves the true token \emph{away} from the predicted mean---\ie away from the tokens it is confused with---weighted by how much it is confused with them, and without preference for any other structure, thus spreading the embeddings apart. A descent step on the first term of the MSE instead moves $\emb(i)$ \emph{towards} $\mu$ with weight $1 - \pi^{\theta,\ell,i}_t$, and translates every confused token $\emb(v)$ with weight $\pi^{\theta,\ell,v}_t$ along the common direction $\emb(i) - \mu$, from the predicted mean to the true token, which brings $\emb(v)$ towards $\emb(i)$ when the confusion concentrates on $v$; on its own, the MSE is trivially minimised by the collapse $W = 0$. The second term has no definite direction.

Again, \emph{the above is only informal}: a trained network merely approximates the posterior.

\subsection{Plaid self-conditioning incorporation}
\label{app:plaid-self-cond-incorp}
With a self-conditioning input signal $\SC$, alongside the usual input $x_t$, the Plaid architecture forwards the two inputs to the rest of the diffusion transformer as
\begin{equation}
    h_0^\ell = \begin{pmatrix}
        \SC & x_t
    \end{pmatrix} \begin{bmatrix}
        W_c \\ W_x
    \end{bmatrix} .
\end{equation}
It thus doubles the size of the effective input matrix (\ie from $d\times h$ to $2d\times h$).

\subsection{Embedded self-conditioning: beyond Euler steps}
\label{app:self-cond-beyond-euler}
\begin{figure}[tbp]
    \centering
    \includegraphics[width=\linewidth]{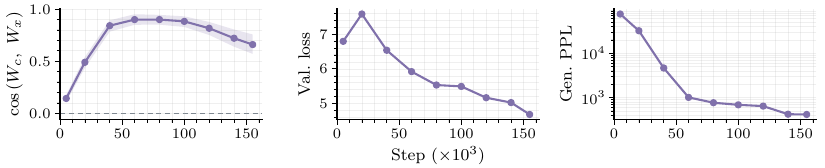}
    \caption{Statistics across training steps, for a Plaid model trained on OWT. \textbf{(Left)} The evolution of cosine similarity between $W_c$ (the part of the input matrix for self-conditioning) and $W_x$ (the part that multiplies the input). \textbf{(Middle)} The validation loss. \textbf{(Right)} The \genppl.}
    \label{fig:alignment-self-cond}
\end{figure}
Another observation which is consistent with our understanding from~\Cref{ssec:self-cond-embedded} is in~\Cref{fig:alignment-self-cond}. Observe that one can split the input projection matrix in two parts: the one that multiplies the usual input, $W_x$, and the one that multiplies the self-conditioning, $W_c$ (see~\Cref{app:plaid-self-cond-incorp} for additional detail).
We can then measure the cosine similarity between their vectorised forms. For the first $80$k steps, the model converges to representations that are highly aligned (with a cosine similarity of $0.9$), and the input projection thus amounts approximately to a linear combination followed by a common projection thereon. After that, the loss, performance and the embedding alignment of the model plateau for roughly $30$k iterations, and then continue improving anew when the cosine similarity drops down to $0.6$, alluding to the fact that the model learns a more complex function than a mere Euler-like step.

\subsection{Self-conditioning and token identifiability}
\label{app:self-cond-identifiability}
We use, here, the path $I_t^\ell = \alpha_t X_0^\ell + \beta_t \emb(X_1^\ell)$ (as in~\Cref{app:vdm-snr}), with $X_0\sim\normal(\vzero, \matid)$ independent of $X_1$, and with $\SCvarl$ being $I_t$-measurable (which holds at train time). Recall the distortion definition:
\begin{equation}
    \distortion_B(R) \coloneqq \sup_{\mathrm{Lip}(g) \leq B} \E\norm{g(R) - g(\emb(X_1^\ell))}^2,
\end{equation}
where $R$ is any square integrable representation of $I_t$. For any $B$-Lipschitz function $g:\R^d\to\R^d$, we have that
\begin{equation}
    \E\norm{g(R) - g(\emb(X_1^\ell))}^2 \leq B^2 \E\norm{R - \emb(X_1^\ell)}^2.
\end{equation}
The linear function $g:x\mapsto Bx$ attains the RHS bound, and thus $\distortion_B(R) = B^2\E\norm{R-\emb(X_1^\ell)}^2$. For the minimisation by $D_t^{\star, \ell}$, observe that
\begin{equation}
    \E\left\langle R - D_t^{\star, \ell}, D_t^{\star,\ell} - \emb(X_1^\ell)\right\rangle = \E\left\langle R - D_t^{\star, \ell},  \E\left[ D_t^{\star,\ell} - \emb(X_1^\ell) \mid I_t \right] \right\rangle = 0,
\end{equation}
since the RHS of the inner product is zero, by definition. The first equality is true by the tower property of the expectation, and $R$ and $D_t^{\star, \ell}$ are both $I_t$-measurable. Therefore, we can add and subtract a $D_t^{\star, \ell}$ term to show that
\begin{equation}
    \label{eq:dist-simplification}
    \distortion_B(R) = B^2\E\norm{ D_t^{\star,\ell} - \emb(X_1^\ell) }^2 + B^2\E\norm{R - D_t^{\star,\ell}}^2,
\end{equation}
which implies minimisation when $R = D_t^{\star, \ell}$ almost surely. For the comparison, define
\begin{equation}
    \epsilon_t^\ell \coloneqq \E\norm{\SCvarl - D_t^{\star, \ell}}^2,
\end{equation}
and
\begin{equation}
    \mmse_t^{\ell} \coloneqq \E\norm{\emb(X_1^\ell) - D_t^{\star, \ell}}^2.
\end{equation}
Therefore,
\begin{equation}
    \distortion_B(I_t^\ell) = B^2\E\norm{\alpha_t X_0^\ell - (1-\beta_t) \emb(X_1^\ell) }^2 = B^2\left[d \alpha_t^2 + (1-\beta_t)^2\E\norm{\emb(X_1^\ell)}^2\right].
\end{equation}
On the other hand, using~\eqref{eq:dist-simplification},
\begin{equation}
    \distortion_B(\SCvarl) = B^2\E\norm{\emb(X_1^\ell) - D_t^{\star, \ell}}^2 + B^2\E\norm{\SCvarl - D_t^{\star, \ell}}^2 = B^2\mmse_t^\ell + B^2\epsilon_t^\ell.
\end{equation}
It follows that
\begin{equation}
    \frac{\distortion_B(I_t^\ell) - \distortion_B(\SCvarl)}{B^2} = d\alpha_t^2 + (1-\beta_t)^2 \E\norm{\emb(X_1^\ell)}^2 - \mmse_t^\ell - \epsilon_t^\ell.
\end{equation}
For the normalisation, $I_t^\ell/\beta_t$, remark that
\begin{equation}
    I_t^\ell/\beta_t - \emb(X_1^\ell) = \frac{\alpha_t}{\beta_t} X_0^\ell + \emb(X_1^\ell) - \emb(X_1^\ell) = \SNR(t)^{-1/2}X_0^\ell.
\end{equation}
Whence:
\begin{equation}
    \distortion_B(I_t^\ell/\beta_t) - \distortion_B(\SCvarl) = B^2\left[ \frac{d}{\SNR(t)} - \mmse_t^\ell - \epsilon_t^\ell \right].
\end{equation}
Note that $\mmse_t^\ell \leq \E\norm{\emb(X_1^\ell)}^2 = 1$ for unit-normalised embeddings.

\paragraph{Dependence on dimension.} We now show that, for unit-normalised embeddings, the MMSE does not depend on $d$ explicitly, but only on $\SNR(t)$, the data distribution and the geometry of the codebook through its Gram matrix $G \in \R^{\vocabsize\times\vocabsize}$, $G_{uv} \coloneqq \langle \emb(u), \emb(v)\rangle$ for $u, v\in V$, \ie $G = WW^\top$.

Let $x_1 = (x_1^1, \ldots, x_1^L) \in V^L$. Since the $X_0^m$ are independent standard Gaussians, the likelihood of $I_t$ given $X_1 = x_1$ is
\begin{equation}
    p(I_t \mid X_1 = x_1) \propto \prod_{m=1}^L \exp\left(-\frac{\norm{I_t^m - \beta_t\emb(x_1^m)}^2}{2\alpha_t^2}\right) \propto \prod_{m=1}^L \exp\left(S^m_{x_1^m}\right), \qquad S^m_v \coloneqq \frac{\beta_t}{\alpha_t^2}\left\langle I_t^m, \emb(v)\right\rangle,
\end{equation}
where the proportionality constants do not depend on $x_1$, as $\norm{\emb(v)}^2 = 1$ for all $v\in V$. Therefore, by Bayes' rule, the posterior $q^\ell_v \coloneqq \Prob(X_1^\ell = v \mid I_t)$ is a function of $S = (S^m_v)_{m, v}$ and of the law of $X_1$ only, and $D_t^{\star,\ell} = \sum_{v\in V}q^\ell_v\emb(v)$. Conditionally on $X_1 = x_1$,
\begin{equation}
    S^m_v = \SNR(t)\, G_{x_1^m v} + \SNR(t)^{1/2}\left\langle X_0^m, \emb(v)\right\rangle,
\end{equation}
where, for each $m$, $(\langle X_0^m, \emb(v)\rangle)_{v\in V}$ is a centred Gaussian vector with covariance $G$, independent across $m$. Hence, the joint law of $(X_1, S)$, and thus that of $(X_1, q^\ell)$, depends only on the law of $X_1$, $\SNR(t)$ and $G$. Finally, writing $\delta_{v}\in\R^\vocabsize$ for the indicator vector of $v$,
\begin{equation}
    \mmse_t^\ell = \E\norm{\sum_{v\in V}\left(\delta_{X_1^\ell} - q^\ell\right)_v \emb(v)}^2 = \E\left[\left(\delta_{X_1^\ell} - q^\ell\right)^\top G \left(\delta_{X_1^\ell} - q^\ell\right)\right],
\end{equation}
which depends on the same quantities only.

Consequently, two codebooks of dimension $d$ and $d'$ sharing a Gram matrix attain the same $\mmse_t^\ell$ at the same $\SNR(t)$, whereas the distortion of the normalised input, $B^2 d/\SNR(t)$, scales linearly with the dimension. That is to say, at a fixed level of difficulty of token identification, the advantage of the posterior mean,
\begin{equation}
    \distortion_B(I_t^\ell/\beta_t) - \distortion_B(\SCvarl) = B^2\left[\frac{d}{\SNR(t)} - \mmse_t^\ell - \epsilon_t^\ell\right],
\end{equation}
grows linearly in $d$ for a fixed approximation error $\epsilon_t^\ell$: the noise lying in the extra directions is a nuisance which $D_t^{\star, \ell}$ averages out, but which $I_t^\ell$ carries into the network.

Observe that the set of Gram matrices of unit-normalised codebooks realisable in dimension $d$ grows with $d$ (any codebook in $\R^d$ embeds into $\R^{d'}$ for $d' > d$ by zero-padding, preserving $G$), so that the lowest $\mmse_t^\ell$ attainable by such a codebook is non-increasing in $d$, whereas the distortion of the normalised input increases linearly in $d$.

\section{Further schedules}
\subsection{Entropic schedule}
\begin{figure}[tbp]
    \centering
    \includegraphics[width=\linewidth]{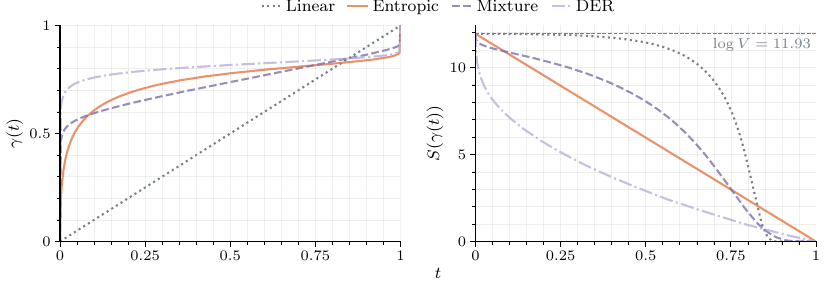}
    \caption{The entropic schedule against the base DER, mixture and linear schedules. \textbf{(Left)} The schedules themselves. \textbf{(Right)} The expected remaining entropy, $S(\gamma_t) = \E_{X_0^\ell}[H(X_1^\ell \mid I_t^\ell)]$.}
    \label{fig:entropic-schedule}
\end{figure}
\begin{figure}[tbp]
    \centering
    \includegraphics[width=\textwidth]{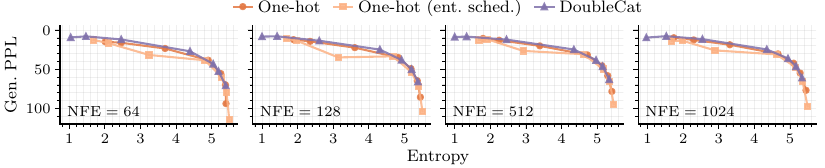}
    \caption{\genppl as measured by GPT-2 for models trained on OpenWebText with sequence lengths of $1024$ tokens, using the GPT-2 tokeniser.}
    \label{fig:entropic-sched-pareto}
\end{figure}
We attempt, here, to understand why the mixture schedule performs better than the raw DER schedule~\citep{fmlm,dfm,scalingcfm}. For the usual one-hot flow case with a uniform prior, the true token-wise Bayesian posterior reads
\begin{equation}
    \forall 0 \leq t \leq 1, \qquad D_t^\ell = \softmax\left(\frac{\gamma_t I_t^\ell}{(1-\gamma_t)^2}\right).
\end{equation}
Let $S:[0,1]\to\R_+$ be the expected remaining entropy; that is to say,
\begin{equation}
    \forall 0 \leq t \leq 1, \qquad S(\gamma_t) \coloneqq \E_{X_0^\ell}\left[H(X_1^\ell \mid I_t^\ell)\right].
\end{equation}
Instead of linearising directly the schedule in the probability of $\argmax~I_t^\ell \neq \argmax~X_1^\ell$ as the DER schedule does, we propose to choose a schedule to advance linearly in $S$. Concretely, we are looking for a solution, $\gamma:[0,1]\to[0,1]$, to the following:
\begin{equation}
    \forall 0 \leq t \leq 1, \qquad S(\gamma_t) = (1-t)\log \vocabsize.
\end{equation}
The solution has no closed form, so we instead build it with a look up table, which incurs a constant cost once before training, exactly as the DER schedule. Let us call the solution the ``entropic'' schedule. We plot the DER, mixture and entropic schedules in~\Cref{fig:entropic-schedule}, alongside the resulting $S$, for each choice of $\gamma$. We remark that the mixture schedule in fact closely follows the entropic schedule both in $\gamma$ and $S \circ\gamma$. Furthermore, we also plot the \genppl/Entropy Pareto fronts of a one-hot flow with the mixture schedule, the entropic schedule and \doublecat in~\Cref{fig:entropic-sched-pareto}. While it closely matches the curve of the mixture schedule, it exhibits an outlier near $3$ nats in entropy, especially at low NFE; however, we report that the resulting performance on TinyGSM is the same as that of the mixture schedule.

We believe that this entropic schedule and the mixture one allude to the fact that making uniform progress in the DER \emph{probability} $P_e$ might be the wrong criterion.

\subsection{Asynchronous schedules}
\label{app:asynchronous-schedules}
While appealing, discrete reveals incur two properties that render the usage of usual continuous diffusion tools difficult: the path is now instead \emph{piecewise} continuous, with discontinuity points that are \emph{stochastic}. For example, the formalism from Flow Map Matching~\citep{flowMapMatching,howToBuildConsistencyModel} then does not port trivially, because of the deterministic and continuous requirements on the flow map.

To return to a continuous and (conditionally on the endpoints) deterministic process, we note that the revealing procedure can be seen as a limiting case of \emph{asynchronous schedules} with added randomness. Instead of having a single time for the evolution of the interpolant, we can have a distinct time for each position, $t^\ell$. %
In principle, a schedule that depends on the current state or the model's denoising distribution is appealing, as parts of the sequence may benefit from being committed earlier than others. Nonetheless, training such a schedule either requires two forward passes from the denoising model, $\pi^\theta$, or training uniformly on the $[0,1]^L$ hypercube to produce a model that is agnostic to the specific choice of schedule; but the latter exposes the model to sequences that have uninformative statistics for inference.

We thus decide to solve this tension by resolving different tokens at different speeds, following \citet{wu2023ardiffusion}. Specifically, local times $t^\ell$ are derived from a global time $t^{g}$, \ie $t^\ell = F(t^g, \ell)$, which is monotonic in $\ell$ and approximately mean-preserving across positions, so that averaging $t^\ell$ over $\ell$ recovers $t^g$. The results are also available in~\Cref{tab:stagger-decoding}, with inner schedule $\kappa_t = t$. Asynchronous times improve accuracy, valid programs and name errors in three of the four matched settings, the exception being the linear schedule with InfoNoise.
InfoNoise raises accuracy almost uniformly, but for the mixture schedule it lowers the fraction of valid programs, so that the best accuracy and the fewest name errors are reached by different configurations. All settings remain well below \semicat (\Cref{tab:tinygsm-doublecat}); we hypothesise that crafting a better time schedule is crucial here.

\paragraph{Implementation.} We follow~\citet{wu2023ardiffusion} and the specific parametrisation by~\citet{bansal2026tokentimecontinuousdiffusion}. We introduce a global time $t^g$ that we sample as before: according to some density $p(t)$ that is often to be taken uniform, but in the case of InfoNoise can be adapted dynamically during training. During inference we step uniformly in $t^g$. For each position $\ell = 1,\ldots,L$ we introduce a local time $t^\ell = F(t^g, \ell)$ that is a function of the global time and the position. Each position $\ell$ gets a fixed rank $r_\ell \in (0,1)$, given by $r_\ell = 1 - (\ell - 0.5)/L$. The rank is then pushed through the quantile function of a Beta distribution whose mean is $t^g$: $t^\ell = Q_t(r_\ell)$, where $Q_t$ is the quantile function of $\mathrm{Beta}(k/(1-t), k/t)$ where $k$ specifies its concentration: as $k\to \infty$ we recover a shared time $t^\ell=t^g$. As $k\to 0$, $\mathrm{Beta}$ tends to a Bernoulli with mean $t^g$, so position $\ell$ switches from noise to data exactly when $t^g$ crosses $(\ell-0.5)/L$: a deterministic left-to-right reveal. In our experiments we use $k=1$; the resulting local times are shown in~\Cref{fig:stagger_field}.

\begin{figure}[t]
    \centering
    \includegraphics[width=\linewidth]{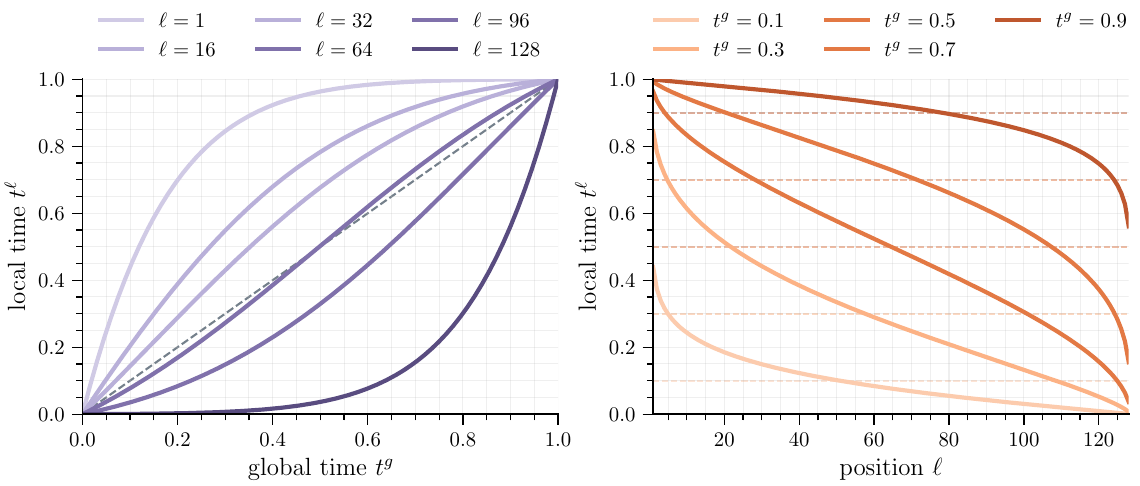}
    \caption{Local per-position times $t^{\ell}$ as a function of global time $t^g$ at concentration $k=1$ and sequence length
$L=128$.
\textbf{(Left)} Each curve traces one position as global time advances.
\textbf{(Right)} Each curve is the spatial profile at one global time, monotone
decreasing in $\ell$: the sequence is revealed left to right. The dashed
horizontals mark the per-position mean, which equals $t^g$ at every setting.
}
    \label{fig:stagger_field}
\end{figure}

\subsection{Adaptive allocation of training density}
Like \citet{hypersphere, cobit} we adapt InfoNoise, the information-guided time allocation of \citet{raya2026noiseschedulinginformationguidedallocation}, to the cross-entropy objective used in our model. The adaptation changes only the distribution $\pi(t)$ from which training times are sampled; the loss weighting itself remains fixed.
Writing $\bar\loss(t)$ for the expected loss at time $t$, the objective thus changes from $\int_0^1 p(t)\bar\loss(t)\,\mathrm dt$ to $\int_0^1 \pi(t)\bar\loss(t)\,\mathrm dt$, where $p$ is the base (uniform) density.

For a fixed reference batch $\mathcal{B}_{\mathrm{ref}}$ of size $|\mathcal{B}_{\mathrm{ref}}| = 256$ (a single batch), we periodically estimate the cross-entropy across a grid of corruption times $\{t_i\}_{i=1}^{M}$,
\begin{equation}
    \hat H(t_i)
    =
    \frac{1}{|\mathcal{B}_{\mathrm{ref}}|}
    \sum_{x\in\mathcal{B}_{\mathrm{ref}}}
    \loss_{\mathrm{CE}}(x,t_i).
\end{equation}
We measure at $M=32$ levels: $24$ at equal mass of the current density, plus 8 spaced uniformly over the interior of $[0,1]$.
Cross-entropy upper-bounds the conditional entropy of the clean token given its corrupted input, with equality for the optimal denoiser, so we use $\hat H(t)$ directly as a proxy for the conditional-entropy profile. This differs from applications of InfoNoise to continuous data, where the information profile must first be estimated indirectly from a denoising-error statistic.

The measured values are projected onto a monotone profile and interpolated with a cubic spline, yielding a smooth estimate $\widetilde H(t)$. We then construct an information-rate profile from its derivative,
\begin{equation}
    r(t)
    =
    \left|
    \frac{\mathrm d \widetilde H(t)}
         {\mathrm dt}
    \right|,
\end{equation}
and allocate training times in proportion to this rate. Thus, corruption regions over which the conditional entropy changes rapidly receive more training samples, while regions in which little information is gained or lost receive proportionally fewer.

To avoid starving regions assigned very low probability, we constrain the resulting density before normalisation. Denoting the unconstrained allocation by $\widetilde\pi(t) \propto r(t)$, we use
\begin{equation}
    \pi^{\mathrm{new}}(t)
    =
    \operatorname{normalise}
    \left[
        \operatorname{clip}
        \left(
            \widetilde\pi(t),
            \pi_{\min},
            \pi_{\max}
        \right)
    \right],
\end{equation}
with $\pi_{\min}=0.1$ and $\pi_{\max}=10$. The lower threshold ensures that every corruption level continues to receive training signal, which is particularly important for an online allocation because a region that is never revisited cannot subsequently develop an estimated information rate that would cause it to regain probability mass. The upper threshold prevents the allocation from becoming excessively concentrated.

We represent the sampling distribution on a fixed grid of 256 points. After each update, the density is numerically integrated on this grid using the trapezoidal rule to obtain a cumulative distribution function $F^{\mathrm{new}}$. Rather than replacing the current allocation immediately, we damp successive updates in CDF space,
\begin{equation}
    F^{(n)}
    =
    (1-\lambda)F^{(n-1)}
    +
    \lambda F^{\mathrm{new}},
    \qquad
    \lambda = 0.5.
\end{equation}
Averaging normalised cumulative distributions rather than raw loss or entropy values makes the update insensitive to changes in the overall scale of the loss during training and smooths changes in where probability mass is allocated. Using a fixed grid likewise provides a common coordinate system across successive measurements.

We recompute the entropy profile and update the allocation every $n_{\mathrm{dynamic}} = 2500$ training iterations. Between updates, corruption times are sampled from the current tabulated CDF by inverse-transform sampling,
\begin{equation}
    u\sim\mathcal U(0,1),
    \qquad
    t = \big(F^{(n)}\big)^{-1}(u),
\end{equation}
where the inverse is obtained by linear interpolation between neighbouring points. The resulting samples determine only which corruption times are presented during subsequent training iterations; all other components of the objective are unchanged.

\begin{table}[t]
  \centering
  \small
  \setlength{\tabcolsep}{4pt}
  \begin{tabular*}{\textwidth}{@{\extracolsep{\fill}}lllcccccc@{}}
    \toprule
    & & & \multicolumn{4}{c}{\textbf{Accuracy} (\%, $\uparrow$)} & \multicolumn{2}{c}{\textbf{Best} (\%)} \\
    \cmidrule(lr){4-7}\cmidrule(lr){8-9}
    InfoNoise & Schedule & Asynchronous & $T = 1$ & $T = 0.5$ & $T = 0.2$ & $\mathrm{top}\text{-}k\text{=}1$ & Runs ($\uparrow$) & Name err. ($\downarrow$) \\
    \midrule
    \multirow{4}{*}{Off} & \multirow{2}{*}{Linear} & Off & $6.82$ & $9.40$ & $11.14$ & $9.40$ & $63.61$ & $36.01$ \\
     &  & On & $\mathbf{7.88}$ & $\mathbf{10.99}$ & $\mathbf{12.28}$ & $\mathbf{10.99}$ & $\mathbf{66.26}$ & $\mathbf{33.43}$ \\
    \cmidrule(l){2-9}
     & \multirow{2}{*}{Mixture} & Off & $10.39$ & $14.78$ & $13.65$ & $13.87$ & $68.16$ & $31.54$ \\
     &  & On & $\mathbf{12.21}$ & $\mathbf{15.47}$ & $\mathbf{16.00}$ & $\mathbf{14.63}$ & $\mathbf{73.39}$ & $\mathbf{26.38}$ \\
    \midrule
    \multirow{4}{*}{On} & \multirow{2}{*}{Linear} & Off & $9.63$ & $17.44$ & $\mathbf{18.42}$ & $15.69$ & $\mathbf{67.78}$ & $\mathbf{31.69}$ \\
     &  & On & $\mathbf{11.37}$ & $\mathbf{18.20}$ & $17.82$ & $\mathbf{17.44}$ & $66.11$ & $33.43$ \\
    \cmidrule(l){2-9}
     & \multirow{2}{*}{Mixture} & Off & $10.84$ & $16.45$ & $14.48$ & $13.80$ & $62.85$ & $36.85$ \\
     &  & On & $\mathbf{15.54}$ & $\mathbf{19.11}$ & $\mathbf{19.41}$ & $\mathbf{18.04}$ & $\mathbf{68.23}$ & $\mathbf{31.31}$ \\
    \bottomrule
  \end{tabular*}
  \caption{
  TinyGSM accuracy at NFE $=256$. For linear and mixture outer schedules $\gamma_t$ and inner schedule $\kappa_t=t$, we compare a scalar time shared across all positions against an asynchronous left-to-right time per token position, each with and without dynamic allocation via InfoNoise. All settings are \doublecat{} sampled as a one-hot flow.
  Better of each matched pair in each column is shown in \textbf{bold}.
  }
  \label{tab:stagger-decoding}
\end{table}